\documentclass[submission, PhysLectNotes]{SciPost}

\usepackage{amsthm}
\newtheorem{definition}{Definition}
\newtheorem{result}{Result}

\usepackage{tikz}
\usepackage{amssymb}

\hypersetup{
    colorlinks,
    linkcolor={red!50!black},
    citecolor={blue!50!black},
    urlcolor={blue!80!black}
}

\usepackage[bitstream-charter]{mathdesign}
\DeclareSymbolFont{usualmathcal}{OMS}{cmsy}{m}{n}
\DeclareSymbolFontAlphabet{\mathcal}{usualmathcal}

\begin{document}

\begin{center}
{\Large \textbf{
Les Houches Lectures on Deep Learning at Large \& Infinite Width
}} \\
\end{center}

\begin{center}
Yasaman Bahri,\footnote{Google DeepMind, Mountain View, CA} Boris Hanin,\footnote{Department of Operations Research \& Financial Engineering, Princeton University} Antonin Brossollet, \footnote{Laboratoire de Physique de l'Ecole Normale Supérieure, Paris, France} Vittorio Erba, \footnote{\'{E}cole polytechnique f\'{e}d\'{e}rale de Lausanne, SPOC and  IdePHIcs Labs} Christian Keup,\footnote{\'{E}cole polytechnique f\'{e}d\'{e}rale de Lausanne, SPOC Lab} Rosalba Pacelli,\footnote{Dipartimento di Scienza Applicata e Tecnologia, Politecnico di Torino and Artificial Intelligence Lab, Bocconi University, Milan} James B. Simon\footnote{Department of Physics, UC Berkeley}
\end{center}




\section*{Abstract}
{\bf
These lectures, presented at the 2022 Les Houches Summer School on Statistical Physics and Machine Learning, focus on the infinite-width limit and large-width regime of deep neural networks. Topics covered include various statistical and dynamical properties of these networks. In particular, the lecturers discuss properties of random deep neural networks; connections between trained deep neural networks, linear models, kernels, and Gaussian processes that arise in the infinite-width limit; and perturbative and non-perturbative treatments of large but finite-width networks, at initialization and after training.\footnote{These are notes from lectures delivered by Yasaman Bahri and Boris Hanin and a first version was compiled by Antonin Brossollet, Vittorio Erba, Christian Keup, Rosalba Pacelli, and James Simon. Recordings of the lecture series can be found at
\url{https://www.youtube.com/playlist?list=PLEIq5bchE3R1QYiNthdj9rJDa4TUzR-Yb}.
}}

\vspace{10pt}
\noindent\rule{\textwidth}{1pt}
\setcounter{tocdepth}{2} 
\tableofcontents\thispagestyle{fancy}
\noindent\rule{\textwidth}{1pt}
\vspace{10pt}

\section{Lecture 1: Yasaman Bahri}

\subsection{Introduction}

This lecture series will be focused on the infinite-width limit and large-width regime of deep neural networks.
Some of the themes that this series will encompass are:
\begin{itemize}
    \item exactly solvable models.
    \item mean-field theory \& Gaussian field theory.
    \item perturbation theory and non-perturbative phenomena.
    \item dynamical systems.
\end{itemize}
Lectures 1-3 are due to Yasaman Bahri and Lectures 4-5 are due to Boris Hanin.


\subsection{Setup}

We are interested in neural networks $f_\theta: \mathbb{R}^{n_0} \to \mathbb{R}^{n_{L+1}}$, where $n_0, n_{L+1}$ are the input and outputs dimensions and $\theta$ denotes the collection of neural network parameters (weights and biases for fully-connected networks, for example). A "vanilla" fully-connected (FC) deep neural network (NN) of hidden layer widths $n_{l}$ and depth $L$ is defined by the iterative relationship
\begin{equation}
    z^l_i(x) = b^l_i + \sum_{j=1}^{n_{l}} W^l_{ij} \phi( z^{l-1}_j(x) )
    \, ,
    \qquad
    1 \leq l \leq L
    \, ,
    1 \leq i \leq n_{l+1}
    \, ,
\end{equation}
and
\begin{equation}
    z^{0}_i(x) = b^{0}_i + \sum_{j=1}^{n_0} W^{0}_{ij} x_j 
    \, , \qquad f_i(x) := z^{L}_i(x), 
\end{equation}
where $x$ is the input, $z^l \in \mathbb{R}^{n_{l+1}}$ is the vector of preactivations at layer $l$, $\{b^l_i, W^l_{ij}\}_{ij}$ are the biases and weights at layer $l$, and $\phi$ is a nonlinear function such as $\tanh$ or ReLU, i.e. $\phi(x) = \max(0, x)$.
The parameters are initialized independently as
\begin{equation} \label{eq.lec1-param_init}
    b^l_i \sim \mathcal{N}(0, \sigma_b^2)
    \, ,
    \qquad
    W^l_{ij} \sim \mathcal{N}\left(0, \frac{\sigma_w^2}{n_{l}}\right),
\end{equation}
where $\mathcal{N}(\mu, \sigma^2)$ is the Normal distribution of mean $\mu$ and variance $\sigma^2$. Note the dependence of the weight variance on the inverse layer width, which will play an important role in subsequent discussions. We refer to the distribution of parameters at initialization as the \textit{prior}. We mainly consider the case of scalar output $n_{L+1} = 1$ in these lectures (the results are straightforward to generalize to the multi-dimensional setting) and uniform hidden layer widths $n_i := n$ for $1 \leq i \leq L$.

\subsection{Prior in function space}

It will be fruitful to translate results, where possible, to the space of functions rather than space of NN parameters, particularly for NNs where there can be a large degree of redundancy in the representation. For example, FC NNs have a permutation symmetry associated with a hidden layer,
\begin{equation}
    W^{l+1}_{ij}, W^l_{jk} \to W^{l+1}_{i\pi(j)}, W^l_{\pi(j)k},
    \qquad
    \forall \text{ permutations } \pi \text{ of } n \text{ elements},
\end{equation}

\noindent so that two different collections of parameters correspond to exactly the same function. A first natural question is then what \textit{prior over functions} is induced by the prior over parameters?

\begin{definition}[Gaussian process]
    A function $f: \mathbb{R}^{n_0} \to \mathbb{R}$ is a draw from a Gaussian process (GP) with mean function $\mu: \mathbb{R}^{n_0} \to \mathbb{R}$ and kernel function $K: \mathbb{R}^{n_0} \times \mathbb{R}^{n_0} \to \mathbb{R}$ if, for any finite collection of inputs $\{x_1, \dots, x_m\}$, the vector of outputs $\{f(x_1), \dots, f(x_m)\}$ is a multivariate Normal random variable with mean 
    $\mu_i = \mu(x_i)$ and covariance $K_{ij} = K(x_i, x_j)$.
\end{definition}

\begin{result}[See Ref. \cite{Neal1996}]
Consider a FC NN with a single hidden layer ($L=1$ in our notation) of width $n$ with parameters drawn i.i.d. as 
\begin{equation}
    b^{0}_i \sim \mathcal{N}(0, \sigma_b^2)
    \, ,
    \quad
    W^{0}_{ij} \sim \mathcal{N}\left(0, \frac{\sigma_w^2}{{n_0}}\right)
    \, ,
    \quad
    b^{1}_i \sim \mathcal{N}(0, \sigma_b^2)
    \, ,
    \quad
    W^{1}_{ij} \sim \mathcal{N}\left(0, \frac{\sigma_w^2}{n}\right).
    \label{eq:one_layer_prior}
\end{equation}
Then, in the limit $n \to \infty$, the distribution of the output $z^{1}_i$, for any $i = 1, ..., n_2$, is a Gaussian process with a deterministic mean function $\mu^{1}(x) = 0$ and kernel function $K^{1}$ given by
\begin{equation}
    K^{1}(x, x') = \sigma_b^2 + \sigma_w^2 \, \mathbb{E}_{u_1, u_2 \sim \mathcal{N}(0,\Sigma)}\left[\phi(u_1)\phi(u_2)\right]
    ,
\end{equation}
where 
\begin{equation}
    \Sigma = \begin{bmatrix}
            K^{0}(x ,x) & K^{0}(x ,x') \\
            K^{0}(x',x) & K^{0}(x',x') 
        \end{bmatrix},
\end{equation}
$K^{0}(x ,x') = \sigma_b^2 + \sigma_w^2 \big( \frac{x \cdot x'}{{n_0}} \big)$, and different outputs $z^1_i, z^1_j$ for $i \neq j$ are independent.

\end{result}
\begin{proof}[Proof (informal).]
Consider the collection of preactivations, $S := \{ z^{1}_i(x_a) \}_{\substack{a=1...m \\ i = 1...n_2}}$, which are random variables conditioned on the input values $x_1, ..., x_m$, and recall that

\begin{equation}
    z^{1}_i(x_a) = b^{1}_i + \sum_{j=1}^n W^{1}_{ij} \phi(z^{0}_j(x_a)) \, .
\end{equation}

Notice that each $z^{1}_i(x_a)$ is a sum of i.i.d. random variables (each, the product of two random variables). By applying the Central Limit Theorem (CLT) to the collection $S$ in the limit of large $n$ and noting that the the variances and covariances are finite, we find that $S$ is governed by the multivariate Normal distribution. Since different outputs $z^1_i, z^1_j$ with $i \neq j$ additionally have zero covariance, they are independent. Below, we will drop the reference to $x_a$ with $a = 1...m$ and instead refer to arbitrary $x, x'$. Bear in mind that the source of randomness is entirely from the parameters and not from the inputs.

The covariance function of the GP for arbitrary $x, x'$ is 
\begin{equation} 
\begin{split}
    \mathbb{E}[ z^{1}_i(x) \,  z^{1}_i(x') ] 
    &= 
    \mathbb{E}[ b_i^{1} b_i^{1} ]
    + \sum_{j, j' = 1}^n \mathbb{E}[ W^{1}_{ij}W^{1}_{ij'} \phi(z^{0}_j(x)) \phi(z^{0}_{j'}(x'))] 
    \\
    &= 
    \sigma_b^2
    + \sum_{j, j' = 1}^n \mathbb{E}[ W^{1}_{ij}W^{1}_{ij'}] \,
    \mathbb{E}[\phi(z^{0}_j(x)) \phi(z^{0}_{j'}(x'))]
    \\
    &= 
    \sigma_b^2 
    + \frac{\sigma_w^2}{n} \sum_{j, j'=1}^n \delta_{jj'} \,
    \mathbb{E}[\phi(z^{0}_j(x)) \phi(z^{0}_{j'}(x'))]
    \\
    &= 
    \sigma_b^2 
    + \frac{\sigma_w^2}{n} \sum_{j=1}^n 
    \mathbb{E}[\phi(z^{0}_j(x)) \phi(z^{0}_j(x'))]
    \\
    &= 
    \sigma_b^2 
    + \sigma_w^2 \,
    \mathbb{E}[\phi(z^{0}_j(x)) \phi(z^{0}_j(x'))]
    \\
    &:= K^{1}(x, x')
    \, ,
\end{split}
\label{eq.lec1-covFirst}
\end{equation}
where the second-to-last line holds for any $j$ and we have used the fact that the contributions from different $j = 1...n$ are identical. We can similarly compute the covariance of the preactivations at the previous layer, obtaining
\begin{equation}
    \mathbb{E}[ z^{0}_i(x) \,  z^{0}_i(x') ] 
    = 
    \sigma_b^2 
    + \sigma_w^2 \bigg( \frac{x \cdot x'}{{n_0}} \bigg)
    := K^{0}(x, x')
    \, .
\end{equation}

Note that the preactivations $z^{0}$ are also described by a multivariate Normal distribution, but in this case it is due to the Normal distribution on the weights and biases since the sum is over ${n_0}$ terms, which we keep finite unlike the hidden layer size $n$. Finally, note the remaining expectation in \eqref{eq.lec1-covFirst} can be expressed as a function of the kernel $K^{0}(x, x')$. Indeed, because of the Gaussianity of $z^{0}$ we have
\begin{equation}
    K^{1}(x, x') = \sigma_b^2 
    + \sigma_w^2 \,
    \mathbb{E}_{u_1, u_2 \sim \mathcal{N}(0,\Sigma)}\left[\phi(u_1)\phi(u_2)\right],
\end{equation}
where 
\begin{equation}
    \Sigma = \begin{bmatrix}
            K^{0}(x ,x) & K^{0}(x ,x') \\
            K^{0}(x',x) & K^{0}(x',x') 
        \end{bmatrix} .
\end{equation}
\end{proof}

\subsection{Prior in function space for deep fully-connected architectures} \label{subsec.lec1-priorFC}

We can generalize this last result to finite-depth FC NNs. There are at least two sensible options for taking the infinite-width limit \cite{dnn_as_gp, g.2018gaussian}:
\begin{itemize}
    \item the \textit{sequential} limit, where the width of each layer $l$ is taken to infinity one by one, from first to last.
    \item the \textit{simultaneous} limit, where the width of each layer $l$ is taken to infinity at the same time.
\end{itemize}
In both cases, with the natural extension of the prior \eqref{eq:one_layer_prior} to multiple layers, each of the hidden-layer preactivations and the output of the NN are again GPs with zero mean and covariance function $K^{l}$ that can be computed iteratively as
\begin{equation}
    K^{l}(x, x') = \sigma_b^2 
    + \sigma_w^2 \, \mathbb{E}_{u_1, u_2 \sim \mathcal{N}(0,\Sigma)}\left[\phi(u_1)\phi(u_2)\right],
\end{equation}
where 
\begin{equation}
    \Sigma = \begin{bmatrix}
            K^{l-1}(x ,x) & K^{l-1}(x ,x') \\
            K^{l-1}(x',x) & K^{l-1}(x',x') 
        \end{bmatrix}
\end{equation}
and the initial covariance is 
$K^{0}(x ,x') = \sigma_b^2 + \sigma_w^2 \bigg( \frac{x \cdot x'}{{n_0}} \bigg)$. We refer readers to the references for proofs of the two cases.

Notice that $\mathbb{E}_{u_1, u_2 \sim \mathcal{N}(0,\Sigma)}\left[\phi(u_1)\phi(u_2)\right]$ is a function of the elements of the covariance matrix $\Sigma \in \mathbb{R}^{2 x 2}$. We will write it generically as 
\begin{equation} \label{eq.lec1-calF}
    \mathcal{F}_\phi(\Sigma_{11}, \Sigma_{12}, \Sigma_{22}) := \mathbb{E}_{u_1, u_2 \sim \mathcal{N}(0,\Sigma)}\left[\phi(u_1)\phi(u_2)\right].
\end{equation}
This function can in fact be computed in closed-form for certain choices of nonlinearity $\phi$.
For the case of ReLU, $\phi = \max(x,0)$, one has
\begin{equation}
    \mathcal{F}_{\rm ReLU}(\Sigma_{11}, \Sigma_{12}, \Sigma_{22})
    = \frac{1}{2\pi} \sqrt{\Sigma_{11}\Sigma_{22}} \left[ 
        \sin \theta + (\pi - \theta) \cos \theta
    \right],
\end{equation}
where $\theta = \arccos{(\Sigma_{12} /  \sqrt{\Sigma_{11}\Sigma_{22}})}$ \cite{cho2009kernel}.

\subsection{Prior in function space for more complex architectures}

The convergence of the prior for wide, deep neural networks to GPs extends to other architectures, such as neural networks with convolutional layers \cite{cnn_as_gp} or attention layers \cite{hron2020infinite}, provided that their weights are initialized i.i.d. with the appropriate inverse scaling of the weight variance with the hidden layer width. The form of the recursion will depend on the nature of the layers.

For example, a simple NN built by stacking one-dimensional convolutional layers is defined by iterating 
\begin{equation}
    z^l_{i, \alpha} = b^l_i + \sum_{j=1}^n \sum_{\beta = -k}^k W^l_{ij,\beta} \phi( z^{l-1}_{j, \alpha + \beta}(x) ),
\end{equation}
where Latin symbols index into \textit{channels} running from $1$ to $n$; Greek symbols on $z$ variables index into \textit{spatial} dimensions running from $1$ to $D$, the spatial dimension of the input; and the index $\beta$ runs over the spatial size $2k+1$ of the convolutional filters. At initialization, we draw parameters i.i.d. as\footnote{As before, this is modified appropriately for the parameters of the first layer, since the input dimension is $n_0$.}
\begin{equation}
    b^{l}_i \sim \mathcal{N}(0, \sigma_b^2)
    \, ,
    \quad
    W^{l}_{ij, \beta} \sim \mathcal{N}\left(0, v_\beta \frac{\sigma_w^2}{{n}}\right)
    ,
\end{equation}
where $v_\beta$ provides a possibly non-uniform magnitude to different spatial coordinates (in the uniform case, $v_\beta = 1/(2k+1)$)\cite{xiao18a, cnn_as_gp}. We take the number of hidden-layer channels $n \rightarrow \infty$ while keeping all other dimensions $k, D, n_0$ fixed. As before, each preactivation and the output of the NN are GPs with zero means, while the covariance function depends on the layer and also acquires spatial components. Indeed
\begin{equation}
\begin{split}
    K^l_{\alpha, \alpha'}(x, x') = \mathbb{E} \left[ z^{l}_{\alpha}(x) z^{l}_{\alpha'}(x') \right]
    &= \sigma_b^2 + \sum_{j, j'=1}^n \sum_{\beta, \beta' = -k}^k 
    \mathbb{E}[W^l_{ij, \beta} W^l_{ij', \beta'}]
    \mathbb{E}[\phi(z^{l-1}_{j, \alpha+\beta}(x) )
    \phi(z^{l-1}_{j', \alpha'+\beta'}(x'))]
    \\
    &= \sigma_b^2 + \sigma_w^2 \sum_{\beta = -k}^k v_\beta \,
    \mathbb{E}[\phi(z^{l-1}_{j, \alpha+\beta}(x) )
    \phi(z^{l-1}_{j, \alpha'+\beta}(x'))]
    \\
    &= \sigma_b^2 + \sigma_w^2 \sum_{\beta = -k}^k v_\beta \,
    \mathcal{F}_{\phi}(K^{l-1}_{\alpha+\beta, \alpha+\beta}(x, x), K^{l-1}_{\alpha+\beta, \alpha'+\beta}(x, x'), K^{l-1}_{\alpha'+\beta, \alpha'+\beta}(x', x')),
\end{split}
\end{equation}
where $\mathcal{F}_{\phi}$ is again the one defined in \eqref{eq.lec1-calF},
and the base case is
\begin{equation}
    K^{0}_{\alpha, \alpha'}(x, x') 
    = \sigma_b^2 + \sigma_w^2  \sum_{\beta = -k}^k v_\beta \frac{1}{{n_0}} \sum_{j=1}^{n_0} x_{j, \alpha + \beta} x'_{j, \alpha' + \beta}.
\end{equation}

Often channel and spatial indices are aggregated into a single index before the output. Below we describe two example strategies; here, $\overline{W}$, $\overline{b}$, $\overline{z}$ refer to the output layer variables.
\begin{itemize}
    \item Aggregation by vectorization  --- In this example, we flatten the last hidden-layer preactivations across channel and spatial dimensions together,
\begin{equation}
    \overline{z}^{L+1}_i(x)
    =
    \overline{b}^{L+1}_i
    + \sum_{j=1}^{n \cdot D} \overline{W}^{L+1}_{ij} \phi(\text{Vec}[z^{L}(x)]_j),
\end{equation}
where $n$, $D$ are the incoming channel and spatial dimensions, respectively, and $\text{Vec}(\cdot)$ is the vectorization operator.
We initialize  $\overline{b}$ as before and $\overline{W}_{ij} \sim \mathcal{N}(0, \frac{\sigma^2_w}{n \cdot D})$.

The covariance of the network output is
\begin{equation}
\begin{split}
    \mathbb{E}[ \overline{z}^{L+1}_i(x) \,  \overline{z}^{L+1}_i(x') ] 
    &= 
    \sigma_b^2
    + \sum_{j, j' = 1}^{n \cdot D} \mathbb{E}[ \overline{W}^{L+1}_{ij}\overline{W}^{L+1}_{ij'}
    ] \, \mathbb{E}[
    \phi(\text{Vec}[z^{L}(x)]_j) 
    \phi(\text{Vec}[z^{L}(x')]_{j'})
    ] 
    \\
    &= 
    \sigma_b^2
    + \frac{\sigma^2_w}{D} \sum_{\alpha = 1}^{D} \mathcal{F}_{\phi} \left(
    K^L_{\alpha,\alpha}(x, x),
    K^L_{\alpha,\alpha}(x, x'),
    K^L_{\alpha,\alpha}(x', x')
    \right).
\end{split}
\end{equation}
In this case, the final covariance depends on the prior layer covariance at the \emph{same} spatial location of two inputs, neglecting some of the information contained in the full tensor $K^{L}_{\alpha, \alpha'}(x,x')$.

\item Aggregation over spatial indices ---
In this example, we aggregate over spatial indices with a fixed vector of weights $h_\alpha$,
\begin{equation}
    \overline{z}^{L+1}_i(x)
    =
    \overline{b}^{L+1}_i
    + \sum_{j=1}^{n} \overline{W}^{L+1}_{ij}
    \sum_{\alpha = 1}^D h_\alpha
    \phi(z^{L}_{j, \alpha}(x)),
\end{equation}
and similar to the previous computations (taking $\overline{W}_{ij} \sim \mathcal{N}(0, \frac{\sigma^2_w}{n})$),
\begin{equation}
    \mathbb{E}[ \overline{z}^{L+1}_i(x) \,  \overline{z}^{L+1}_i(x') ] 
    = 
    \sigma_b^2
    + \sigma^2_w \sum_{\alpha, \alpha' = 1}^{D}
    h_\alpha h_{\alpha'}
    \mathcal{F}_{\phi} \left(
    K^L_{\alpha,\alpha}(x, x),
    K^L_{\alpha,\alpha'}(x, x'),
    K^L_{\alpha',\alpha'}(x', x')
    \right).
\end{equation}
Notice that in this case, even with spatially uniform aggregation $h_\alpha = 1/D$, the final covariance receives spatially off-diagonal contributions from the prior layer covariance.

\end{itemize}

Finally, we note that residual NNs are another architecture that is straightforward to treat. The preactivations take the form
\begin{equation}
    z^l_i(x) = 
    b^l_i + \sum_{j=1}^n W^l_{ij} \phi( z^{l-1}_j(x) ) + \gamma^l z^{l-1}_i(x),
\end{equation}
where $\gamma^l$ are fixed hyperparameters.
In this case, the kernel recursion takes the form
\begin{equation}
    K^{l}(x, x') = \sigma_b^2 
    + \sigma_w^2 \, \mathcal{F}_{\phi} \left(
    K^{l-1}(x, x),
    K^{l-1}(x, x'),
    K^{l-1}(x', x')
    \right)
    + (\gamma^l)^2 \, K^{l-1}(x, x') \, .
\end{equation}

To summarize, we have seen how compositional kernels and GPs can emerge from taking a natural infinite-width limit of deep NNs in different architectural classes. The quantities we have derived
\begin{itemize}
    \item can be used directly in kernel ridge regression or Bayesian inference. In some settings, these kernel-based predictors can be as good as or better models than their NN counterparts.
    \item enable further theoretical understanding of deep NNs at initialization and after training. As one example, understanding the structure of these compositional kernels on realistic data can lend insight into the advantages of different architectures.
\end{itemize}

\subsection{Bayesian inference for Gaussian processes}

Consider a dataset $\mathcal{D} = \{(x_i, y_i)\}_{i=1\dots m}$ and suppose we would like to make predictions at a point $x_*$ in a Bayesian manner, using a model $f_\theta(x)$ with learnable parameters $\theta$. Let $\vec{x} = [x_1, \dots, x_m]^T$ and $\vec{y} = [y_1, \dots, y_m]^T$. The distribution of the output $z_* = f_\theta(x_*)$, conditioned on the dataset $\mathcal{D}$ and $x_*$, is given by
\begin{equation}
    p(z_* \mid \mathcal{D}, x_*) =
    \int d\theta \,
    p(z_* \mid \theta, x_*) \,
    p(\theta \mid \mathcal{D}) \, .
\end{equation}

A convenient way to rewrite this is to introduce the vector of function values on the training data, $\vec{z} = [f_\theta(x_1), \dots f_\theta(x_m)]$. Then
\begin{equation}\label{eq.sec1-z*}
    p(z_* \mid \mathcal{D}, x_*) =
    \int d\vec{z} \,
    p(z_* \mid \vec{z}, \vec{x}, x_*) \,
    p(\vec{z} \mid \mathcal{D}),
\end{equation}
in which we changed the integral over parameters to an integral over the finite set of function values. 

A natural question is under which conditions the conversion from parameter to function space is allowed. In general, one might expect a functional integral over functions that can be represented by the model, i.e. $\int \mathfrak{D}z$. In our case, we are implicitly assuming that the likelihood depends on the parameters only through the outputs of the model. Note that working in function space might allow certain properties of the model to be constrained more naturally, such as function smoothness; on the other hand, other forms of regularization (such as $L_2$ regularization on parameters) might be more challenging to write in a simple form.

We would like to now consider a specific likelihood. By Bayes' theorem,
\begin{equation}
    p(\vec{z} \mid \mathcal{D}) 
    \rightarrow p(\vec{z} \mid \vec{y}) = \frac{
        p(\vec{y} \mid \vec{z}) p(\vec{z}) 
    }{p(\vec{y})},
\end{equation}
(we forgo writing the conditioning on inputs where it is understood), and assuming the targets and model are related by zero-mean Gaussian noise of variance $\sigma^2_{\epsilon}$,
\begin{equation}
    p(\vec{y} \mid \vec{z}) \propto \prod_{i=1}^m \exp\left[ - \, \frac{(y_i - z_i)^2 }{2 \sigma^2_\epsilon} \right] .
\end{equation}

The terms $p(z_* \mid  \vec{z})$ and $p(\vec{z})$ combine to yield the prior distribution $p(z_*, \vec{z})$, which for GPs is a multivariate Gaussian distribution with mean and covariance that depend on the inputs $(x_*, \vec{x})$. Assuming zero mean, we have
\begin{equation}
    p(z_*, \vec{z}) \propto \exp\left\{ - \frac{1}{2}
        \begin{bmatrix}
            z_* & \vec{z} 
        \end{bmatrix}
        \begin{bmatrix}
            K(x_* ,x_*) & K( \vec{x} ,x_*)^T \\
            K( \vec{x} ,x_*) & K( \vec{x} , \vec{x}) 
        \end{bmatrix}^{-1}
        \begin{bmatrix}
            z_* \\ \vec{z}
        \end{bmatrix}
    \right\},
\end{equation}
where $K(\vec{x} ,x_*)_i = K( x_i ,x_*)$ is an $m$-dimensional column vector and $K(\vec{x} ,\vec{x})_{ij} = K( x_i ,x_j)$ is an $m\times m$-dimensional matrix.

We see that the predictive distribution \eqref{eq.sec1-z*} of the model output $z_*$ at $x_*$ involves an integral with a Gaussian integrand, and thus $z_* | \mathcal{D}, x_* \sim \mathcal{N}(\mu_*, \sigma^2_*)$ with
\begin{equation}
\begin{split}
    \mu_* &= K( \vec{x} ,x_*)^T \left( K( \vec{x} , \vec{x})  + \sigma^2_\epsilon I \right)^{-1} \vec{y} \, , \\
    \sigma^2_* &= K( x_*, x_*) - K( \vec{x} ,x_*)^T \left( K( \vec{x} , \vec{x})  + \sigma^2_\epsilon I \right)^{-1} K( \vec{x} ,x_*) .
    \label{eq:bayesian_inference_gp}
\end{split}
\end{equation}
The marginal likelihood $p(\mathcal{D})$ for GPs can be expressed analytically as 
\begin{equation}
    \log p(\mathcal{D}) =
    - \frac{1}{2} \vec{y}^T
    \left( K( \vec{x} , \vec{x})  + \sigma^2_\epsilon I \right)^{-1}
    \vec{y}
    - \frac{1}{2} \log \det  \left( K( \vec{x} , \vec{x})  + \sigma^2_\epsilon I \right) 
    - \frac{m}{2} \log 2 \pi .
\end{equation}
Here the first term accounts for dataset fitting, while the second represents a complexity penalty that favors simpler covariance functions.

In contrast to Bayesian inference for generic models, which often requires approximations because of the integrals involved, Bayesian inference for GPs \cite{williams2006gaussian} can be performed exactly. Given the ``NNGP" \cite{dnn_as_gp} correspondence between infinitely-wide NNs and GPs discussed in prior sections, we can use the resulting compositional kernels to make Bayesian predictions using deep NNs in this limit. 

\subsection{Large-depth fixed points of Neural Network Gaussian Process (NNGP) kernel recursion}

We would now like to investigate the large-depth behavior $l \to \infty$ of the NNGP kernel recursion 
\begin{equation}
\begin{split}
    K^{l}(x, x') &= \sigma_b^2 
    + \sigma_w^2 \, \mathcal{F}_\phi(
    K^{l}(x, x),
    K^{l}(x, x'),
    K^{l}(x', x')
    ), \\
    K^{0}(x, x') &= \sigma_b^2 
    + \sigma_w^2 \bigg( \frac{x \cdot x'}{{n_0}} \bigg).
\end{split}
\end{equation}

As training deep NNs was known to be a challenge in practice, these large-depth limits have been used \cite{schoenholz2017} as proxy metrics to identify regions of hyperparameter space where networks can be trained. (In this example, hyperparameters for which we might desire guidance on choosing include $\sigma_w$, $\sigma_b$, $L$, and $\phi$.) It was hypothesized that for deep NNs to be trainable using backpropagation, forward propagation of information about the inputs through the depth of the network would be needed. In lieu of an information-theoretic approach, a proxy for the information content contained in the forward signal is the covariance between pairs of inputs. Regions of hyperparameter space where the covariance function quickly converges to a structureless limit are to be avoided for choosing architectures and initialization strategies. We will briefly treat the simplest analysis (of forward propagation) in this direction for the case of a fully-connected NN \cite{schoenholz2017}. With further developments in deep learning theory, analogous but comprehensive treatments have been constructed; we refer the reader to the later literature, see e.g. \cite{xiao18a, roberts2021principles, xiao2020disentangling}.

Let us consider the correlation between a pair of inputs $x_\alpha$ and $x_\beta$. We will need to track recursions for the three quantities $K_{\alpha\alpha}$, $K_{\beta\beta}$, and $K_{\alpha\beta}$. For the diagonal elements,
\begin{equation}
    K_{\alpha\alpha}^l = \sigma_b^2 
    + \sigma_w^2 \int Ds \, \bigg( \phi(\sqrt{K^{l-1}_{\alpha\alpha}} s) \bigg)^2,
\end{equation}
where $Ds$ is the standard Gaussian measure, while for off-diagonal elements
\begin{equation}
\begin{split}
    K_{\alpha\beta}^l &= \sigma_b^2 
    + \sigma_w^2 \int Ds_1 \, Ds_2 \, \phi(u_1) \phi(u_2), \end{split}
\end{equation}
with
\begin{equation}
\begin{split}
    u_1 &= \sqrt{K^{l-1}_{\alpha\alpha}} s_1,\\
    u_2 &= \sqrt{K^{l-1}_{\beta\beta}} \left( c^{l-1}_{\alpha\beta} s_1 + \sqrt{ 1- (c^{l-1}_{\alpha\beta})^2} \, s_2  \right), \\
    c^{l}_{\alpha\beta} &= K_{\alpha\beta}^{l} / \sqrt{K_{\alpha\alpha}^{l} K_{\beta\beta}^{l} }.
\end{split}
\end{equation}

Now suppose that the diagonal elements of the kernel approach a fixed point $q^*$ (this occurs for any bounded $\phi$ and the convergence is rapid with depth, see \cite{schoenholz2017}). In this case, note that $c^* = 1$ is always a fixed point of the recursion for off-diagonal covariances, as we recover the condition for the fixed point of the diagonal elements. Is the fixed point stable or unstable to leading order in small deviations? By expanding the map $c^{l-1}_{\alpha\beta} \to c^l_{\alpha\beta}$ around the fixed point, one finds the stability of $c^* = 1$ is governed by 
\begin{equation}
    \chi_1 = \frac{\partial c^l_{\alpha\beta} }{ \partial c^{l-1}_{\alpha\beta}} = \sigma_w^2 \int Ds \, \bigg( \phi'(\sqrt{q^*}s) \bigg)^2.
\end{equation}
If $\chi_1 < 1$, then $c^* =1$ is a stable fixed point, while if $\chi_1 > 1$ it is unstable.

\begin{figure}
    \centering
    \includegraphics[scale =.4]{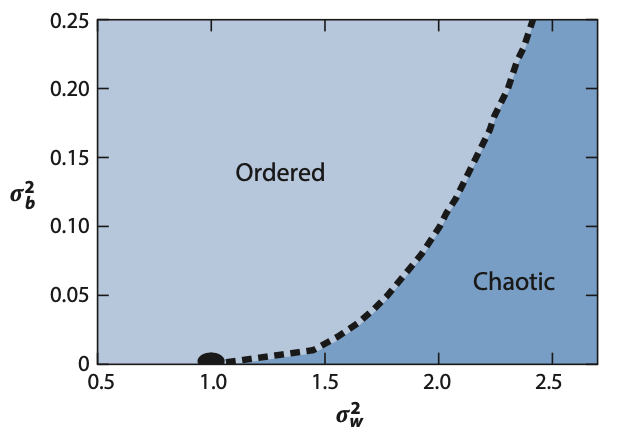}
    \caption{Phase diagram in the $(\sigma_b^2, \sigma_w^2)$ plane for fixed points of the NNGP recursion relationship with nonlinearity $\phi = \tanh$, showing \emph{ordered} and \emph{chaotic} phases separated by a critical line. Figure reproduced from \cite{bahri2020}; see also \cite{schoenholz2017}.}
    \label{fig:ordered_chaotic}
\end{figure}

The rate of convergence with depth can be obtained by expanding the recursion relationships to leading order around the fixed points \cite{schoenholz2017}. In the case of the diagonal elements, we define
$\epsilon^l := K^l_{\alpha\alpha} - q^*$
and obtain
\begin{equation}
    \epsilon^l = \epsilon^{l-1} \left[ 
        \chi_1 + \sigma_w^2 \int Ds \, \phi''(\sqrt{q^*} s)
        \phi(\sqrt{q^*} s)
    \right] + O((\epsilon^{l-1})^2)
\end{equation}
so that, at large $l$, $\epsilon^l = \epsilon^0 \exp\left( - l / \xi_q \right)$ with characteristic depth scale
\begin{equation}
    \xi_q^{-1} = - \log \left[ 
        \chi_1 + \sigma_w^2 \int Ds \, \phi''(\sqrt{q^*} s)
        \phi(\sqrt{q^*} s)
    \right].
\end{equation}
To study off-diagonal elements, we instead examine the correlation $\epsilon^l = c^{l}_{\alpha \beta} - c^*$. On the basis that the diagonal elements approach their fixed point $q^*$ more rapidly \cite{schoenholz2017}, we substitute $K^{l}_{\alpha \alpha} = q^*$ to obtain
\begin{equation}
    \epsilon^l = \epsilon^{l-1} \left[ 
        \sigma_w^2 \int Ds_1 \, Ds_2 \,  \phi'(u_1^*)
        \phi'(u_2^*)
    \right] + O((\epsilon^{l-1})^2),
\end{equation}
where
\begin{equation}
\begin{split}
    u_1^* &= \sqrt{q^*} s_1 \, ,\\
    u_2^* &= \sqrt{q^*} \left( c^* s_1 + \sqrt{ 1- (c^*)^2} s_2  \right) \, ,
\end{split}
\end{equation}
and the characteristic depth is given by
\begin{equation}
    \xi_c = - \log \left[ 
        \sigma_w^2 \int Ds_1 \, Ds_2 \,  \phi'(u_1^*)
        \phi'(u_2^*)
    \right].
\end{equation}

\noindent We have now three options:
\begin{itemize}
    \item if $\chi_1 < 1$, then $c^* = 1$ is a stable fixed point, and we term the corresponding region of the $(\sigma_b, \sigma_w)$ plane the \textit{ordered phase}.
    In this phase, on average across random networks two inputs $x_\alpha$, $x_\beta$ will tend to align exponentially fast, with characteristic depth $\xi_c$, as they propagate through layers of the deep NN.
    \item if $\chi_1 > 1$, then $c^* = 1$ is an unstable fixed point, and the corresponding region of the $(\sigma_b, \sigma_w)$ plane is termed a \textit{chaotic phase}. There will be another fixed point $c^* < 1$ which will be stable. Two inputs $x_\alpha$, $x_\beta$ will tend towards uniform correlation (possibly vanishing) across all pairs $\alpha \neq \beta$ exponentially fast in the NN depth, with a characteristic depth scale $\xi_c$.
    \item if $\chi_1 = 1$, then $c^* = 1$ is marginally stable, and stability is determined by higher-order terms in the expansion around the fixed point. The corresponding region of the $(\sigma_b, \sigma_w)$ plane is a \textit{critical line}.
    In this phase, the correlation between two inputs $x_\alpha$, $x_\beta$ will tend towards a fixed point at a slower rate, algebraically instead of exponentially fast. Indeed, one can show that as $\chi_1 \to 1$, $\xi_c \to +\infty$. It is found that the maximum depth of a NN that can be trained with backpropagation increases as the initialization hyperparameters get closer to this critical line \cite{schoenholz2017}. 
\end{itemize}

Let us consider the case $\phi = \tanh$ as an example, with phase diagram in Fig. \ref{fig:ordered_chaotic} showing ordered and chaotic phases separated by a critical line. The ordered phase is smoothly connected to the regime $\sigma_b \gg \sigma_w$; intuitively, the shared bias dominates over the weights acting on the input signals, and two inputs degenerate into a common value as they are passed through deeper layers of the random network (hence, the stability of the $c^* =1$ fixed point). The chaotic phase smoothly connects to the regime $\sigma_b \ll \sigma_w$, where randomness from the weights dominates and leads to reduced correlation between the inputs.

\newcommand{\todo}[1]{\textcolor{blue}{\textbf{TODO:} #1}}
\newcommand{\expec}[1]{\mathbb{E}\left[ #1 \right]}
\newcommand{\mcom}{\ ,}
\newcommand{\mdot}{\ .}
\newcommand{\dd}{\mathrm{d}}
\newcommand{\normin}[1]{\left\lVert #1 \right\rVert_{p^{in}}}
\newcommand{\norm}[1]{\left\lVert #1 \right\rVert}
\newcommand{\bigo}[1]{O \left( #1 \right)}
\newcommand{\lin}{^{\text{lin}}}

\section{Lecture 2}

\subsection{Introduction}

In the previous lecture, we treated the properties of deep neural networks at initialization in the limit $n \rightarrow \infty$. We also discussed the predictions arising from Bayesian inference in this limit. In this lecture, we turn our attention to training deep NNs with empirical risk minimization and understanding the optimization dynamics, either by gradient descent or gradient flow, in this same limit of infinitely-wide hidden layers. 

Before doing so, we introduce a few tools that enable us to analytically treat leading deviations away from the infinite-width limit in randomly initialized deep NNs. These tools have also been used to construct a perturbation theory for finite-width deep NNs after training \cite{roberts2021principles}.


\subsection{Wick's theorem}

Wick's theorem is a fundamental result about Gaussian random variables that simplifies computations involving expectations of products of such variables. 

\begin{result}[Wick's theorem]
    Let $z \in \mathbb{R}^n$ be a centered random Gaussian vector with covariance matrix $K$, $z \sim \mathcal{N} \left( 0, K \right)$. 
    Then, the expectation of any product of the elements of $z$ can be expressed as a sum over all possible pairings of indices
    \begin{equation}
        \expec{ z_{\mu_1} \dots z_{\mu_{2m}} } = \sum_{\text{all pairings}} \expec{z_{\mu_{k_1}} z_{\mu_{k_2}} } \dots \expec{ z_{\mu_{k_{2m-1}}} z_{\mu_{k_{2m}}} } = \sum_{\text{all pairings}} K_{\mu_{k_1} \mu_{k_2}} \dots K_{\mu_{k_{2m-1}} \mu_{k_{2m}}}.
        \label{eq:wicksthm}
    \end{equation}
\end{result}

\noindent (Here, the result is for products containing an even number of elements since odd ones vanish.) We will use this to compute higher-order correlation functions in randomly initialized deep linear networks, illustrating some of the effects of finite-width which carry beyond deep linear networks to nonlinear ones \cite{roberts2021principles}.

\subsubsection{Two-point correlation function}
Using Wick's theorem (\ref{eq:wicksthm}), we compute the covariance between any two preactivations of the same layer for a randomly initialized deep linear neural network, assuming no bias terms for simplicity \cite{roberts2021principles},
\begin{equation} 
\begin{split}
    \expec{z_{i_1}^l (x_{\alpha}) \ z_{i_2}^l (x_{\beta}) } &=  \sum_{j_1, j_2 = 1}^n \expec{ W_{i_1 j_1}^{l} W_{i_2 j_2}^{l} z_{j_1}^{l-1} (x_{\alpha}) \ z_{j_2}^{l-1} (x_{\beta}) } \\
    &= \sum_{j_1, j_2 = 1}^n \expec{ W_{i_1 j_1}^l W_{i_2 j_2}^l} \expec{z_{j_1}^{l-1} (x_{\alpha}) \ z_{j_2}^{l-1} (x_{\beta}) } \\
    &=  \delta_{i_1 i_2} \frac{\sigma_w^2}{n} \sum_{j_1, j_2 = 1}^n \delta_{j_1 j_2} \expec{z_{j_1}^{l-1} (x_{\alpha}) \ z_{j_2}^{l-1} (x_{\beta}) } \\
    &=  \delta_{i_1 i_2} \frac{\sigma_w^2}{n} \sum_{j = 1}^n  \expec{z_{j}^{l-1} (x_{\alpha}) \ z_{j}^{l-1} (x_{\beta}) }.
\end{split}
\label{eq.lec2-wicks_twopts}
\end{equation}

\noindent Let us decompose the two-point correlation function for inputs $x_{\alpha}, x_{\beta}$ in layer $l$ as
\begin{equation}
    \expec{z_{i_1}^l (x_{\alpha}) \ z_{i_2}^l (x_{\beta})} := \sigma_w^2\delta_{i_1 i_2} G_{\alpha \beta}^{l},
\end{equation}
where $G_{\alpha \beta}^{l}$ is defined as
\begin{equation}
    G_{\alpha \beta}^{l} = \frac{1}{n} \sum_{j=1}^n \expec{z_{j}^{l} (x_{\alpha}) \ z_{j}^{l} (x_{\beta}) }.
\end{equation}
With these definitions, we can express the recursion \eqref{eq.lec2-wicks_twopts} in a compact form
\begin{equation}
    G_{\alpha \beta}^{l} = \sigma_w^2 G_{\alpha \beta}^{l-1},
\end{equation}
leading to the depth-dependent form
\begin{equation}
    G_{\alpha \beta}^{l} = \left( \sigma_w^2 \right)^l G_{\alpha \beta}^{0}.
\end{equation}

\subsubsection{Four-point correlation function}

Similarly, we obtain a recursion for the four-point correlation function,

\begin{equation} \label{eq.lec2-wicks_fourpts}
\begin{split}
    \expec{z_{i_1}^{l} \dots \ z_{i_4}^{l} } &= \sum_{j_1 \dots j_4 = 1}^n \expec{ W_{i_1 j_1}^{l} \dots W_{i_4 j_4}^{l} } \expec{z_{j_1}^{l-1} \dots \ z_{j_4}^{l-1}} \\
    &= \frac{\left( \sigma_w^2 \right)^2}{n^2} \sum_{j_1 \dots j_4 = 1}^n \left( \delta_{i_1 i_2} \delta_{j_1 j_2} \delta_{i_3 i_4}  \delta_{j_3 j_4} + \delta_{i_1 i_3} \delta_{j_1 j_3} \delta_{i_2 i_4} \delta_{j_2 j_4} + \dots \right) \expec{z_{j_1}^{l-1} z_{j_2}^{l-1} z_{j_3}^{l-1} z_{j_4}^{l-1} } \\ 
    &=  \left( \sigma_w^2 \right)^2 \left( \delta_{i_1 i_2} \delta_{i_3 i_4} + \delta_{i_1 i_3} \delta_{i_2 i_4} +\delta_{i_1 i_4} \delta_{i_2 i_3} \right) \frac{1}{n^2} \sum_{j,k=1}^n \expec{z_{j}^{l-1} z_{j}^{l-1} z_{k}^{l-1} z_{k}^{l-1} }  \mcom
\end{split}
\end{equation}
using Wick's theorem to decompose the expectation value into a sum over pairings of indices. (For simplicity, we have treated the case of a single sample $x_{\alpha} = x$ and dropped reference to the samples, but this calculation can be extended to a general choice of four samples.)

We again factor the correlation function as a term encoding the structure of indices and a scalar function,
\begin{equation}
    \expec{z_{i_1}^{l} \dots \ z_{i_4}^{l} } := \left( \delta_{i_1 i_2} \delta_{i_3 i_4} + \delta_{i_1 i_3} \delta_{i_2 i_4} +\delta_{i_1 i_4} \delta_{i_2 i_3} \right) G_4^{l} \mdot
\end{equation}
Using this decomposition the final factor of the recursion \eqref{eq.lec2-wicks_fourpts} can be written as 
\begin{equation}
    \frac{1}{n^2} \sum_{j,k=1}^n \expec{z_{j}^{l-1} z_{j}^{l-1} z_{k}^{l-1} z_{k}^{l-1}}  = \frac{1}{n^2} \sum_{j,k=1}^n \left( \delta_{jj} \delta_{kk} + \delta_{jk} \delta_{jk} +\delta_{jk} \delta_{jk} \right) G_4^{l-1} = \left( 1 + \frac{2}{n} \right) G_4^{l-1},
\end{equation}

\noindent which yields the recursion 
\begin{equation}
    G_4^{l} = \left( \sigma_w^2 \right)^2 \left( 1 + \frac{2}{n} \right) G_4^{l-1}.
\end{equation}
It is easy to see using \eqref{eq.lec2-wicks_fourpts} that 
\begin{equation}
    G_4^{0} = \left( G_2^{0} \right) ^2,
\end{equation}
with 
\begin{equation}
    G_2^{0} = \frac{\sigma^2_w}{n_0} \sum_{j=1}^{n_0} x_j x_j.
\end{equation}

\noindent referring to the two-point correlation function. Unrolling the recursion relation we obtain various relationships
\begin{equation} \label{eq.lec2-fourptsrecursion}
\begin{split}
    G_4^{l} &= \left( \sigma_w^2 \right)^{2l} \left[ \prod_{l'=1}^{l} \left( 1 + \frac{2}{n} \right) \right] \left( G_2^{0} \right)^2 \\
    &= \left[ \prod_{l'=1}^{l} \left( 1 + \frac{2}{n} \right) \right] \left( G_2^{l} \right)^2 \\
    &=\left( 1 + \frac{2}{n} \right)^{l} \left( G_2^{l} \right)^2.
\end{split}
\end{equation}

 
\subsubsection{Large-$n$ expansion}

Let us discuss what we learn from these simple applications of Wick's Theorem \cite{roberts2021principles}. In the limit $n \to \infty$, the recursion for \eqref{eq.lec2-fourptsrecursion} simplifies to $G_4^{l} = \left( G_2^{l} \right)^2$
and the correlation function becomes
\begin{equation}
    \expec{z_{i_1}^l z_{i_2}^l z_{i_3}^l z_{i_4}^l} =\left( \delta_{i_1 i_2} \delta_{i_3 i_4} + \delta_{i_1 i_3} \delta_{i_2 i_4} +\delta_{i_1 i_4} \delta_{i_2 i_3} \right) \left( G_2^{l} \right)^2 \mcom
\end{equation}
which is what we would obtain if all the preactivations were Gaussian random variables (indeed, we know from the last lecture the preactivations are described by a Gaussian process in this limit). Large but finite $n$ gives rise to a deviation from Gaussianity which to leading order acquires the form
\begin{equation}
\begin{split}
     G_4^{l} - \left( G_2^{l} \right)^2 &= \left[ \left( 1 + \frac{2}{n} \right)^{l} - 1 \right] \left( G_2^{l} \right)^2 \\
     =& \frac{2 l}{n} \left( G_2^{l} \right)^2 + O \left( \frac{1}{n^2} \right),
\end{split}
\end{equation}
valid if the depth is not too large. The correction to the four-point correlation function from its infinite-width form is therefore governed by the ratio of the depth to width of the network, $l/n$. It turns out this ratio also governs the corrections to gradient-based learning in trained finite-width deep NNs\cite{roberts2021principles}. The deviations from Gaussianity at finite width will be discussed further in Lectures 4 and 5.

\subsection{Gradient descent dynamics of optimization in the infinite-width limit}

We next treat the dynamics of training deep NNs within empirical risk minimization under gradient flow (GF) or gradient descent (GD) in the infinite-width limit. We specialize to the case of square loss, where an analytic closed-form derivation is possible. This setting further develops the rich set of connections between infinitely-wide neural networks, kernel regression, and Gaussian processes \cite{jacot2018, leexiao2019wide} which we partly established in the first lecture.


\subsubsection{Setting}
We consider a fully-connected deep NN of depth $L$ and width $n$ represented by $f_t(x) : \mathbb{R}^{n_0} \to \mathbb{R}$ with parameters $\theta_t = \left\{ W_{ij}^l (t), b_i^l (t)\right\}_{lij}$. We view the NN function and parameters as inheriting a time dependence from optimization and use the notation $f_t(x) = f(x, \theta_t)$ to emphasize this. The loss function on a dataset $\mathcal{D} = \left\{ (x_{\alpha}, y_{\alpha}) \right\}_{\alpha=1}^m$ is

\begin{equation}
    \mathcal{L} (\theta) = \frac{1}{m} \sum_{\alpha = 1}^m \ell \left( f(x_{\alpha} , \theta), y_{\alpha} \right),
\end{equation}
where we take $\ell$ to be the square loss. We will sometimes write $\mathcal{L}_t = \mathcal{L}(\theta_t)$.

\subsubsection{Gradient descent dynamics for the neural network function}
\label{sec:chain_rule_function_space}

Let us investigate the dynamics on the deep NN function that arises from applying gradient descent to the network parameters. The latter are updated as
\begin{equation}
    \theta_{\mu, t+1} = \theta_{\mu,t} - \eta \frac{\partial \mathcal{L}_t}{\partial \theta_{\mu}},
\end{equation}
where $\mu$ indexes into the collection of trainable parameters and $\eta$ is the learning rate. (In what follows, we will use the $\mu$ index where necessary when it clarifies the structure of contracted derivatives, but will drop it otherwise.) Transcribing from parameter to function space dynamics using the chain rule, we can expand in the limit of small learning rate
\begin{equation}
\begin{split}
    f_{t+1}(x) = f(x, \theta_{t+1}) &= f(x, \theta_t - \eta \frac{\partial \mathcal{L}_t}{\partial \theta}) \\
    &= f_t(x) - \eta \sum_{\mu} \frac{\partial f_t(x)}{\partial \theta_{\mu}} \frac{\partial \mathcal{L}_t}{\partial \theta_{\mu}} + \frac{\eta^2}{2} \sum_{\mu, \nu}  \frac{\partial^2 f_t(x)}{\partial \theta_{\mu} \partial \theta_{\nu}} \frac{\partial \mathcal{L}_t}{\partial \theta_{\mu}} \frac{\partial \mathcal{L}_t}{\partial \theta_{\nu}} + \dots.
\end{split}
\end{equation}

For illustration, we will examine the continuous time limit of these dynamics, but they can straightforwardly be extended to the discrete time setting by keeping higher-order terms in $\eta$. Letting $\eta$ tend to zero, the evolution of the function $f_t(x)$ under gradient flow is
\begin{equation}
    \frac{\dd f_t(x)}{\dd t} = - \sum_{\mu} \frac{\partial f_t(x)}{\partial \theta_{\mu}} \frac{\partial \mathcal{L}_t}{\partial \theta_{\mu}}.
\end{equation}

\noindent It will be useful to separate out the relationship between the NN function and the parameters (this map encodes the structure of the NN) with the relationship between the NN function and the loss, rewriting the gradients as 
\begin{equation}
    \frac{\partial \mathcal{L}_t}{\partial \theta_{\mu}} = \sum_{\alpha \in \mathcal{D}} \frac{\partial \mathcal{L}_t}{\partial f_t(x_{\alpha})} \frac{\partial f_t(x_{\alpha})}{\partial \theta_{\mu}}.
\end{equation}

\noindent The function evolution can therefore be rewritten 
\begin{equation}
\begin{split}
     \frac{\dd f_t(x)}{\dd t} &= - \sum_{\alpha \in \mathcal{D}} \frac{\partial \mathcal{L}_t}{\partial f(x_{\alpha})}  \left[ \sum_{\mu}   \frac{\partial f_t(x_{\alpha})}{\partial \theta_{\mu}} \frac{\partial f_t(x)}{\partial \theta_{\mu}} \right] \\
     &= - \sum_{\alpha \in \mathcal{D}} \frac{\partial \mathcal{L}_t}{\partial f(x_{\alpha})} \Theta_t(x_{\alpha}, x).
     \label{eq:level1_eq}
\end{split}
\end{equation}

\noindent The quantity $\Theta_{t}(x,x')$ is the inner product defined as
\begin{equation}
    \Theta_t(x,x') := \sum_{\mu} \frac{\partial f_t(x)}{\partial \theta_{\mu}}  \frac{\partial f_t(x')}{\partial \theta_{\mu}} \mdot
\end{equation}

The $n \rightarrow \infty$ limit of these dynamics was first studied in \cite{jacot2018}, where the quantity $\Theta_{t}(x,x')$ was introduced. The infinite-width limit of $\Theta$ in randomly initialized networks was termed the \emph{Neural Tangent Kernel} (NTK), and it will play a crucial role in our subsequent discussion. We will overload terminology and in what follows also use this term to refer to the dynamical variable $\Theta_t(x,x')$, which may be evaluated away from initialization or in finite-sized networks.

For the form of the loss we consider, the term $\partial \mathcal{L}_t/\partial f(x)$ only depends on $f_t(x)$. However, $\Theta_t(x,x')$ is, in general, a new variable whose dynamics we need to track to ensure a closed system of equations. Time derivatives arise entirely from the dynamics of parameters, so we can make the substitution for the operator
\begin{equation}
    \frac{\dd}{\dd t}  = \sum_{\mu} \frac{\partial \theta_{\mu}}{\partial t} \frac{\partial}{\partial \theta_{\mu}}.
\end{equation}

\noindent The time evolution of the NTK is
\begin{equation}
\begin{split}
    \frac{\dd \Theta_{t}(x,x')}{\dd t} &= \sum_{\mu} \left[ \frac{\dd}{\dd t} \left( \frac{\partial f_t(x)}{\partial \theta_{\mu}} \right) \frac{\partial f_t(x')}{\partial \theta_{\mu}} + \frac{\partial f_t(x)}{\partial \theta_{\mu}} \frac{\dd}{\dd t} \left( \frac{\partial f_t(x')}{\partial \theta_{\mu}} \right) \right] \\
    &= \sum_{\mu, \nu} \left[ \frac{\partial \theta_{\nu}}{\partial t} \frac{\partial}{\partial \theta_{\nu}} \left( \frac{\partial f_t(x)}{\partial \theta_{\mu}} \right) \frac{\partial f_t(x')}{\partial \theta_{\mu}} + \frac{\partial f_t(x)}{\partial \theta_{\mu}} \frac{\partial \theta_{\nu}}{\partial t}  \frac{\partial}{\partial \theta_{\nu}} \left( \frac{\partial f_t(x')}{\partial \theta_{\mu}} \right) \right] \\
    &= - \sum_{\mu, \nu} \left[ \frac{\partial \mathcal{L}_t}{\partial \theta_{\nu}} \frac{\partial}{\partial \theta_{\nu}} \left( \frac{\partial f_t(x)}{\partial \theta_{\mu}} \right) \frac{\partial f_t(x')}{\partial \theta_{\mu}} + \frac{\partial f_t(x)}{\partial \theta_{\mu}} \frac{\partial \mathcal{L}_t}{\partial \theta_{\nu}}  \frac{\partial}{\partial \theta_{\nu}} \left( \frac{\partial f_t(x')}{\partial \theta_{\mu}} \right) \right] \\
    &= - \sum_{\alpha \in \mathcal{D}} \frac{\partial \mathcal{L}_t}{\partial f(x_{\alpha})} \left[ \sum_{\mu, \nu} \bigg( \frac{\partial^2 f_t(x)}{\partial \theta_{\mu} \partial \theta_{\nu}} \frac{\partial f_t(x_{\alpha})}{\partial \theta_{\nu}} \frac{\partial f_t(x')}{\partial \theta_{\mu}} + \frac{\partial^2 f_t(x')}{\partial \theta_{\mu} \partial \theta_{\nu}} \frac{\partial f_t(x_{\alpha})}{\partial \theta_{\nu}} \frac{\partial f_t(x)}{\partial \theta_{\mu}} \bigg) \right],
    \label{eq:level2_eq}
\end{split}
\end{equation}
where from the second line to the third line, we used the gradient flow of the parameters.

We find that the dynamics of the NTK is therefore governed by the quantity in square brackets in Eq. \ref{eq:level2_eq}, which involves a contraction of first and second-order derivatives of the network map $\theta \rightarrow f(x)$. The quantity in square brackets may be a new dynamical variable different from $f, \Theta$ in general, and the full set of equations describing the dynamics in function space is generically not closed at this level (that is, using only \ref{eq:level1_eq} and \ref{eq:level2_eq}), requiring the generation of further equations as we did for $\Theta$. We will return to this procedure in the next lecture. 

\subsubsection{Remark on normalization}
\label{normalization}
     Constructing and training a neural network comes with a design choice as to the parameterization of the parameters, and their initialization prior to optimization. Our presentation has thus far followed historical development; to maintain consistency with the subsequent literature, we now use a different parameterization and initialization, which in the literature has been referred to as \emph{NTK parameterization} \cite{jacot2018, leexiao2019wide}. In the layer-to-layer transformations, we factor out an explicit $\sigma_w/\sqrt{n_l}$ dependence in front of the weights (the important factor is the size-dependence rather than $\mathcal{O}(1)$ constants) and instead use the initialization scheme $W_{ij}^l \sim \mathcal{N}(0, 1)$. In contrast, the scheme used thus far absorbs the appropriate width scaling into the initialization, e.g. $W_{ij}^{l} \sim \mathcal{N}(0, \sigma^2_w / n_l)$; it is referred to as \emph{standard parameterization} in the literature and has been a common empirical practice in deep learning. Both schemes give rise to the same Gaussian processes in the infinite-width limit, but their dynamics under gradient descent for general networks can be somewhat different. The choice of parameterization also affects the dependence of a suitable choice of hyperparameters, such as the learning rate in gradient descent, on the size of the network. (In standard parameterization, the learning rate will exhibit an implicit dependence on the size of hidden layers.) Hereafter, we will use NTK parameterization since it makes the infinite-width limiting behavior more explicit; however, the important features of our discussion (notably the connection between the infinite-width limit, kernel regression, and Gaussian processes) will be unaffected and so we do this without loss of generality. Further discussion on the difference between NTK and standard parameterization can be found in \cite{leexiao2019wide}. We note in passing that, since this earlier work, a rich literature has developed on the topic of parameterizations, hyperparameter selection, and infinite-width limits.

\subsubsection{Example: single hidden-layer neural network}\
\label{sec:single_layer}
Let us examine a concrete example: a single hidden-layer NN. We set the biases to zero for simplicity. The preactivations in the hidden layer and the output are 
\begin{equation}
\begin{split}
    z_i^0(x) &= \sum_{k=1}^{n_0} \sigma_w \frac{W_{ik}^0}{\sqrt{{n_0}}} x_k \\
    f(x) &= \sum_{i=1}^{n} \sigma_w \frac{W_i^1}{\sqrt{n}} \phi \left( z_i^0(x) \right). 
\end{split}
\end{equation}
The Neural Tangent Kernel for this network is
\begin{equation}
    \Theta(x,x') = \frac{\sigma^2_w}{n} \sum_{i=1}^n \phi \left( z_i^0(x) \right) \phi \left( z_i^0(x') \right) + \bigg( \frac{\sigma^2_w}{n} \sum_{i=1}^n \left( W_i^1 \right)^2 \phi' \left( z_i^0(x) \right) \phi' \left( z_i^0(x') \right) \bigg) \left( \frac{\sigma^2_w}{{n_0}} \sum_{j=1}^{n_0} x_j x'_j \right) \mdot
\end{equation}

\noindent The training dynamics of the preactivations obey
\begin{equation}
    \frac{\dd z_i^0(x)}{\dd t} = - \sum_{\alpha \in \mathcal{D}} \frac{\partial \mathcal{L}_t}{\partial f(x_{\alpha})} \frac{\sigma_w W_i^1 \phi' \left( z_i^0(x_{\alpha}) \right)}{\sqrt{n}} \left( \frac{\sigma^2_w}{{n_0}} \sum_{j=1}^{n_0} x_{\alpha, j} x_j \right),
\end{equation}
and the dynamics of the weights in the last layer are
\begin{equation}
    \frac{\dd W_i^1}{\dd t} = - \sum_{\alpha \in \mathcal{D}} \frac{\partial \mathcal{L}_t}{\partial f(x_{\alpha})} \frac{\sigma_w \phi \left( z_i^0(x_{\alpha}) \right)}{\sqrt{n}}.
\end{equation}
A back-of-the-envelope estimate shows that at initialization, these two quantities vanish as $n \rightarrow \infty$; indeed $\left| \frac{\dd z_i^0}{\dd t} \right|_{t=0} \sim O(\frac{1}{\sqrt{n}})$ and $\left| \frac{\dd W_i^1}{\dd t} \right|_{t=0} \sim O(\frac{1}{\sqrt{n}})$. These terms contribute to the dynamics of $\Theta_t$, and a simple calculation suggests that a similar vanishing occurs for the NTK evolution estimated at initialization, as $n \rightarrow \infty$,
\begin{equation}
    \left| \frac{\dd \Theta_t(x,x')}{\dd t} \right| _{t=0} \xrightarrow[n \to \infty]{} \ 0.
\end{equation}

\noindent On the other hand, we calculated $\Theta_t$ above, and it is $\mathcal{O}(1)$ at initialization as $n \rightarrow \infty$. Thus, a suggestive picture based on these estimates at initialization as $n \rightarrow \infty$, \emph{assuming they continue to hold true during training}, is that individual parameters and hidden-layer preactivations do not evolve under dynamics in this limit, and the NTK remains at its initial value.

\subsubsection{Neural Tangent Kernel in the infinite-width limit}

We will give a physicist's treatment of the behavior of the NTK in the infinite-width limit, both at initialization and after training. The NTK is in general a random variable when the network parameters are themselves drawn from a distribution. However, certain properties become deterministic due to the law of large numbers as $n \rightarrow \infty$. The first main result states that the NTK at initialization approaches a deterministic quantity as $n \rightarrow \infty$, with a recursion relation that parallels the recursion we derived in Lecture 1 for the NNGP. The second main result considers the dynamics of the NTK as $n \rightarrow \infty$: surprisingly, the NTK stays constant during the course of training. Both of these results were hinted at in the last section for a single hidden-layer NN, based off of our back-of-the-envelope estimates at initialization: specifically, we found that $\Theta_{t=0} \sim \mathcal{O}(1)$ and $d\Theta_{t=0}/dt \rightarrow 0$.

The constancy of the NTK enables us to solve for the network evolution analytically and gives a connection between deep NN learning and kernel regression. 

\paragraph{Initialization}

\begin{result}[\cite{jacot2018}]\label{res.lec2-ntkconvergence}
    For a network of depth $L$ at initialization with nonlinearity $\phi$, and in the limit as the layer width $n \to \infty$ sequentially, the NTK $\Theta^{L}$ converges to a deterministic limiting kernel: 
    \begin{equation}
        \Theta^{L, kj} \to \Theta_{0}^{L} \cdot \delta_{kj},
    \end{equation}
where we treat the general setting of non-scalar maps $f: \mathbb{R}^{n_0} \rightarrow \mathbb{R}^{n_{L+1}}$,
    \begin{equation}
        \Theta^{L, kj} (x, x') := \sum_{\mu} \frac{\partial f_k(x)}{\partial \theta_{\mu}} \frac{\partial f_j(x')}{\partial \theta_{\mu}},
    \end{equation}
\end{result}

\noindent and $\Theta^{L}_{0}: \mathbb{R}^{n_0} \times \mathbb{R}^{n_0} \rightarrow R$ is a kernel whose recursion relation we will derive below, while touching on the main ideas behind the proof of this result; we refer the reader to \cite{jacot2018} for complete technical details. We can understand how the result arises by induction in the depth of the network. As we sequentially take each hidden layer to be of infinite size, we will leverage the fact that the distribution of preactivations originating from that layer is described by a Gaussian process with a covariance function given by the NNGP kernels discussed in Lecture 1. 

To derive the recursion relation, we split the parameters into two groups corresponding to those from the last layer and those from earlier in the network.
\begin{equation}
\begin{split}
    \Theta^{L, kj}(x,x') &= \sum_{\mu} \frac{\partial f_k(x)}{\partial \theta_{\mu}} \frac{\partial f_j(x')}{\partial \theta_{\mu}} \\
    &= \sum_{\substack{\mu \in \text{ last layer} \\ L}} \frac{\partial f_k(x)}{\partial \theta_{\mu}} \frac{\partial f_j(x')}{\partial \theta_{\mu}} + \sum_{\substack{\mu \in \text{ earlier layers} \\ 1, \dots, L-1}} \frac{\partial f_k(x)}{\partial \theta_{\mu}} \frac{\partial f_j(x')}{\partial \theta_{\mu}}.
\end{split}
\end{equation}
Working in NTK parameterization for the layer-to-layer transformation,
\begin{equation}
    f_k(x) = z_k^L(x) = \sigma_b b_k^L + \sum_{i=1}^n \sigma_w \frac{W_{ki}^L}{\sqrt{n}} \phi \left( z_i^{L-1}(x) \right),
\end{equation}
the NTK takes the form
\begin{multline}
     \Theta^{L, kj}(x,x') = \delta_{kj} \sigma_b^2 + \delta_{kj} \frac{\sigma_w^2}{n} \sum_{i=1}^n \phi \left( z_i^{L-1}(x) \right) \phi \left( z_i^{L-1}(x') \right) \\ + \delta_{kj} \sigma^2_w  \sum_{i,s =1}^n \frac{W_{ki}^L W_{js}^L}{n} \phi' \left( z_i^{L-1}(x) \right) \phi' \left( z_s^{L-1}(x') \right) \sum_{\substack{\mu \in \text{ earlier layers} \\ 1, ..., L-1}} \frac{\partial z_i^{L-1}(x)}{\partial \theta_{\mu}} \frac{\partial z_s^{L-1}(x')}{\partial \theta_{\mu}}.
\end{multline}
Using the induction hypothesis we can simplify the last term 
\begin{multline}
    \Theta^{L, kj}(x,x') = \delta_{kj} \bigg[ \sigma_b^2 + \frac{\sigma_w^2}{n} \sum_{i=1}^n \phi \left( z_i^{L-1}(x) \right) \phi \left( z_i^{L-1}(x') \right) \\ + \frac{\sigma_w^2}{n} \sum_{i =1}^n (W_{ki}^L)^2 \phi' \left( z_i^{L-1}(x) \right) \phi' \left( z_i^{L-1}(x') \right) \Theta^{L-1}(x,x') \bigg].
\end{multline}
The second and third term are averages of i.i.d. random variables in the infinite width limit. Thus, by the law of large numbers, they concentrate to their mean when $n \to \infty$. Since the distribution on $z^{L-1}$ is given by a Gaussian process, we can further simplify the expression. Revisiting the discussion in \ref{subsec.lec1-priorFC}, we let
\begin{equation}
\begin{split}
    \mathcal{F}_{\phi} (\Sigma) &= \mathbb{E}_{(u,v) \sim \mathcal{N}(0, \Sigma)} \left[ \phi(u) \phi(v) \right] \\
    \widetilde{\mathcal{F}}_{\phi} (\Sigma) &= \mathbb{E}_{(u,v) \sim \mathcal{N}(0, \Sigma)} \left[ \phi'(u) \phi'(v) \right],
\end{split}
\end{equation}
where 
\begin{equation}
\Sigma = \left( \begin{array}{c|c}
    K_{11} & K_{12} \\ 
    \hline
    K_{21} & K_{22}
\end{array} \right).
\end{equation}
\noindent The first and second terms concentrate to 
\begin{equation}
\begin{split}
    \sigma_b^2 + \sigma_w^2 \: \expec{\phi \left( z_i^{L-1}(x) \right) \phi \left( z_i^{L-1}(x') \right)} &= \sigma_b^2 + \sigma_w^2 \: \mathcal{F}_{\phi} (K^{L-1}(x,x), K^{L-1}(x,x'), K^{L-1}(x',x')) \\
    &= K^L(x,x').
\end{split}
\end{equation}

\noindent The third term as $n \rightarrow \infty$ becomes
\begin{equation}
\begin{split}
    \sigma_w^2 \: \expec{\left( W_{ki}^L \right)^2} \expec{\phi' \left( z_i^{L-1}(x) \right) \phi' \left( z_i^{L-1}(x') \right) } \, \Theta^{L-1}(x,x') &= \sigma_w^2 \widetilde{\mathcal{F}}_{\phi}(K^{L-1}(x,x), \dots) \, \Theta^{L-1}(x,x').
\end{split}
\end{equation}

Altogether, in a randomly initialized infinitely-wide deep NN, we have the following recursion for the NTK,
\begin{equation}
    \Theta^{L, kj}(x,x') = \delta_{kj} \bigg( K^L(x,x') + \sigma_w^2 \widetilde{\mathcal{F}}_{\phi}(K^{L-1}(x,x), K^{L-1}(x,x'), K^{L-1}(x',x')) \cdot \Theta^{L-1}(x,x') \bigg).
\end{equation}

Hence the NTK depends both on the two-point correlation function of forward-propagated signal (i.e. $K^L$) and on back-propagated signal (such as the integral involving the derivative of $\phi$, which can sometimes be computed in closed-form). 

\paragraph{Training} \mbox{} \\ 

In examining the single-hidden layer NN in Sec. \ref{sec:single_layer}, we saw how the size-dependent parameterization (or initialization) factors of $1/\sqrt{n}$ resulted in dynamical variables such as individual weight matrix elements or individual pre-activations in a layer acquiring a vanishingly small rate of change, with respect to optimization time, at initialization as $n \rightarrow \infty$. This originated from the combination of inverse-$n$ dependent factors and other quantities remaining $\mathcal{O}(1)$; it then resulted in the vanishing of the time derivative of the NTK at initialization. In fact, with certain losses (such as square loss as we are considering), this vanishing time derivative continues to hold during training \cite{jacot2018}, so that macroscopic variables such as $\Theta_t(x,x')$, as well as individual parameters and preactivations, are frozen at their initial values in the infinite-width limit. (While these dynamical variables stay at their initial values during optimization as $n \rightarrow \infty$, they collectively still enable the NN function to adapt and fit the training data). To summarize this informally,

\begin{result}[~\cite{jacot2018}]
    Under gradient flow on the mean-squared error, as $n \to \infty$, the NTK stays constant during training and equal to its initial value,
    \begin{equation}
        \Theta^{L, kj}_t \to \Theta^{L}_{0} \delta_{kj}.
    \end{equation}
\end{result}

\noindent Consequently, the differential equation for the NN function takes the simple form
\begin{equation}
    \frac{\dd f_t(x)}{\dd t} = - \sum_{\alpha \in \mathcal{D}} \frac{\partial \mathcal{L}_t}{\partial f(x_{\alpha})} \Theta_0(x_{\alpha}, x) = - \sum_{\alpha \in \mathcal{D}} \bigg( f_t(x_{\alpha}) - y_{\alpha} \bigg) \, \Theta_0(x_{\alpha}, x),
   \label{eq.2:dynoutNTK}
\end{equation}

\noindent which can be solved exactly.

\subsubsection{Closed-form solution for dynamics and equivalent linear model}
\label{sec:equivalent_linear_model}

We can derive an explicit solution for $f_t(x)$ from the linear ordinary differential equation in \eqref{eq.2:dynoutNTK}. Before doing so, however, we discuss an equivalent formulation for the dynamics that lends perspective to the complexity of the model that is learned in this infinite-width limit and yields a parameter-space description. As we state in the next section (Result \ref{res:linear_conv}), the optimization dynamics of the NN function in the infinite-width limit is equivalent to the function realizing a first-order Taylor expansion with respect to the NN parameters \cite{leexiao2019wide}; more precisely, it realizes the specific linear model  

\begin{equation}
    f_t\lin(x) := f_0(x) + \nabla_{\theta}f_0(x)\big|_{\theta=\theta_0} \cdot \omega_t \mcom
    \label{eq:firstorderTaylor}
\end{equation}
where $\omega_t = \theta_t - \theta_0$ is the change in the parameters during training from their initial value. Note that this model is still nonlinear with respect to inputs $x$. Hence, we can also study parameter space dynamics in the infinite-width limit, obtaining a linear ODE for the NN parameters $\theta_t$ in analogy to \eqref{eq.2:dynoutNTK}. Let $\mathcal{X}$ and $\mathcal{Y}$ denote the collection of training inputs and targets vectorized over the sample dimension $m = 1, ..., M$. Solving the ODEs in closed-form yields 

\begin{align}
    &\omega_t = - \nabla_{\theta}f_0(\mathcal{X})^\top \cdot \Theta^{-1}_0 \cdot \left( I - e^{-\Theta_0 t} \right) \cdot (f_0(\mathcal{X}) - \mathcal{Y}), \\
    &f_t\lin(\mathcal{X}) = \left( I - e^{-\Theta_0 t} \right) \mathcal{Y} + e^{-\Theta_0 t} f_0(\mathcal{X}),
\end{align}
where $\Theta_0 \equiv \Theta_0(\mathcal{X}, \mathcal{X})$. The value of the NN function in the infinite-width limit (equivalently, the value of the linear model \eqref{eq:firstorderTaylor}) is 
\begin{equation}
    f_t (x) = \underbrace{\Theta_0(x,\mathcal{X}) \cdot \Theta_0^{-1} \cdot \left(I - e^{-\Theta_0 t} \right) \cdot \mathcal{Y}}_{\mu_t(x)} + \underbrace{f_0(x) - \Theta_0(x,\mathcal{X}) \cdot  \Theta_0^{-1} \cdot \left(I - e^{-\Theta_0 t} \right) \cdot f_0(\mathcal{X})}_{\gamma_t(x)}, 
    \label{eq:exactly_solvable_general_x}
\end{equation}
where we grouped all the terms depending on the initial function in $\gamma_t(x)$. While we have solved the dynamics for a particular instantiation of an infinite-width random network, if we consider the distribution on $f_t(x)$ that arises from the initial distribution on $f_0$ (namely, $f_0(x)$ is a sample from a GP), we find $f_t(x)$ is also described by a GP whose mean and covariance functions can be calculated from \eqref{eq:exactly_solvable_general_x}. (We separated the terms into $\mu_t(x)$ and $\gamma_t(x)$ to hint that they contribute to the mean and variance of $f_t(x)$, respectively.) This GP can be contrasted with the one arising from Bayesian inference in the infinite-width limit \eqref{eq:bayesian_inference_gp}. The GP arising from empirical risk minimization and gradient flow has mean and variance \cite{leexiao2019wide}

\begin{equation}
\begin{split}
    \mu(x) &= \Theta_0(x, \mathcal{X}) \cdot \Theta^{-1}_0 \cdot \mathcal{Y}, \\
    \sigma^2(x) &= K(x,x) + \Theta_0(x, \mathcal{X}) \cdot \Theta^{-1}_0 \cdot K \cdot \Theta^{-1}_0 \cdot \Theta_0(\mathcal{X}, x) - 
                 ( \Theta_0(x, \mathcal{X}) \cdot \Theta^{-1}_0 \cdot K(\mathcal{X}, x) + \\
                                & K(x, \mathcal{X}) \cdot \Theta^{-1}_0 \cdot \Theta_0(\mathcal{X}, x) ).
\end{split}
\end{equation}

\noindent (Recall that $\Theta_0, K$ without arguments refers to the $m \times m$ matrix constructed by evaluating on training samples $\mathcal{X}$.)

Rather surprisingly, we have found that the infinite-width limit under optimization leads to exactly solvable dynamics for deep neural networks. In principle, the result could have been quite complicated, and with infinitely-many parameters the learned function might have been rather ill-behaved. Instead, the dynamics have a relatively simple description: it is captured by the kernel $\Theta_0$ associated with the deep NN and which is computable via recursion relations. We reiterate how this simplicity came about due to the way in which NN parameters are commonly represented (either through explicit or implicit factors involving the hidden-layer size) in deep learning.

\subsubsection{Aside: linear model equivalence in two parameterizations}

In Sec. \ref{sec:equivalent_linear_model}, we mentioned how gradient flow dynamics at infinite width realizes a linear relationship between the NN function and parameters during the course of training \eqref{eq:firstorderTaylor}, and that this is equivalent to the Neural Tangent Kernel $\Theta$ staying constant at its initial value $\Theta_0$ as $n \rightarrow \infty$. While we have focused our discussion on gradient flow, these equivalences between infinite-width deep NN dynamics, linear models, kernel regression, and Gaussian processes hold under gradient descent up to a maximum learning rate. Below, we state these results informally \cite{leexiao2019wide}, highlighting the value of the maximum learning, and contrast how the results appear in NTK and standard parameterization.

\begin{result}[\cite{leexiao2019wide}]\label{res:linear_conv}
    Assume that the smallest eigenvalue of the NTK at initialization is positive $\lambda_{\text{min}} > 0$ and let $\lambda_{\text{max}}$ be the largest eigenvalue. Under gradient descent with a learning rate $\eta < \eta_{\text{critical}}$ where $\eta_{\text{critical}} = \frac{2}{\lambda_{\text{min}} + \lambda_{\text{max}}}$,  we have (in NTK parameterization),
    \begin{equation}
    \begin{split}
        \sup_{t \geq 0} \norm{f_t(x) - f_t^{\text{lin}}(x)}_2 &= \bigo{\frac{1}{\sqrt{n}}} \\
        \sup_{t \geq 0} \frac{\norm{\theta_t - \theta_0}_2}{\sqrt{n}} &= \bigo{\frac{1}{\sqrt{n}}} \quad \qquad \text{as } n \to \infty \\
        \sup_{t \geq 0} \norm{\Theta_t - \Theta_0}_F &= \bigo{\frac{1}{\sqrt{n}}}.
    \end{split}
    \end{equation}
\end{result}

In standard parametrization (c.f. \ref{normalization}), it is necessary to have $\eta_0 < \eta_{\text{critical}}$ and the learning rate used in gradient descent is instead $\eta := \eta_0/n$. In this parameterization, we define the Neural Tangent Kernel as 
\begin{equation}
    \Theta = \frac{1}{n} \sum_{\mu} \frac{\partial f(x)}{\partial \theta_{\mu}} \frac{\partial f(x')}{\partial \theta_{\mu}},
\end{equation}
and the analogous scalings are
\begin{equation}
\begin{split}
    \sup_{t \geq 0} \norm{f_t(x) - f_t^{\text{lin}}(x)}_2 &= \bigo{\frac{1}{\sqrt{n}}} \\
    \sup_{t \geq 0} \norm{\theta_t - \theta_0}_2 &= \bigo{\frac{1}{\sqrt{n}}}  \quad \qquad \text{as } n \to \infty \\
    \sup_{t \geq 0} \norm{\Theta_t - \Theta_0}_F &= \bigo{\frac{1}{\sqrt{n}}}.  
\end{split}
\end{equation}

\noindent We see that the primary differences between the two parameterizations in the infinite-width limit is the bound on the $L_2$ parameter distance moved during optimization and the form of the maximum learning rate.

\def\sumData{\sum_{x^{\alpha}\in \mathcal{D}}}   
\def\xalpha{x^{\alpha}}                          
\def\xbeta{x^{\beta}}
\def\loss{\mathcal{L}}                           
\def\bigO{\mathcal{O}}                           
\def\todo#1{\textcolor{blue}{\textbf{TODO: #1}} \\}  
\newcommand{\pderiv}[2]{\frac{\partial #1}{\partial #2}}                     
\newcommand{\secondpderiv}[3]{\frac{\partial^2 #1}{\partial #2 \partial #3}} 
\def\Ralpha{R^{\alpha}}       
\def\Othree{\mathbb{O}_3}
\def\Ofour{\mathbb{O}_4}
\def\O#1{\mathbb{O}_{#1}}
\newcommand{\E}[2]{\mathbb{E}_{#2}\left[ #1 \right]}  
\def\T{\mathrm{T}}  

\section{Lecture 3}

\subsection{Introduction}

In this lecture, we go beyond the exactly solvable infinite-width limit to discuss both perturbative and non-perturbative corrections that are visible at large but finite width. One aspect of the exactly solvable limit discussed in Lecture 2 is that it exhibits no ``feature learning"; rather, the model relies on a fixed set of random features from initialization for prediction. Equivalently, the Neural Tangent Kernel does not change during the course of training. Finite-size hidden layers in a deep neural network instead gives rise to ``weak" or ``strong" amounts of feature learning, and one goal of this lecture is to illustrate two theoretical descriptions of such feature learning.

We begin by revisiting the function space description we alluded to in Sec. \ref{sec:chain_rule_function_space} which gives rise to a hierarchy of coupled differential equations necessary for closure. This hierarchy can be truncated to compute leading order corrections arising from finite width \cite{DyerGurAri2020, huangYau}.\footnote{While we do not discuss it here, capturing the effect of depth is treated in \cite{roberts2021principles}.} We then give a contrasting example of a minimal model whose learning (at large $n$) is quite different than the exactly solvable kernel limit and its perturbative corrections, a phenomenon termed "catapult dynamics" \cite{lewkowycz2020large}. This phenomenon arises from using a learning rate in gradient descent that is larger than the critical value (Result \ref{res:linear_conv}).

\subsection{Perturbation theory for dynamics at large but finite width}\label{sec:perturbative_NTK}

In Sec. \ref{sec:chain_rule_function_space}, we derived ODEs for the evolution of the NN function and the dynamical Neural Tangent Kernel under gradient flow. (Here, we use the abbreviation $R_{\alpha} := \frac{\partial \mathcal{L}_t}{\partial f(x_\alpha)} = f_t(x_\alpha) - y_\alpha$ for the residual originating from the loss.) These were
\begin{equation}
    \frac{df_t(x)}{dt} = 
    -\sum_{\alpha \in \mathcal{D}}  \underbrace{\frac{\partial \mathcal{L}_t}{\partial f(x_{\alpha})}}_{ \equiv R_{\alpha, t}} \, \Theta_t (x_{\alpha}, x),
\end{equation}
and
\begin{equation}\label{derivative_of_NTK}
    \frac{d \Theta_t(x,x')}{dt} = 
         - \sum_{\alpha \in \mathcal{D}} R_{\alpha, t} \underbrace{
                            \left[ 
                            \sum_{u,v} \pderiv{f_t(x_{\alpha})}{\theta_u} 
                            \secondpderiv{f_t(x)}{\theta_u}{\theta_v}
                            \pderiv{f_t(x')}{\theta_v} 
                            + (x \leftrightarrow x').
                            \right]
                            }_{\mathbb{O}_3 (x,x',x_{\alpha})}.
\end{equation}

\noindent where we use $(x \leftrightarrow x')$ to denote the expression obtained from exchanging $x$ and $x'$ in the preceding term appearing in square brackets (hence, note that $\mathbb{O}_3$ is symmetric under exchange of arguments $x, x'$).

\subsubsection{Hierarchy of coupled ODEs}
\label{sec:coupledODEhierarchy}

While the specific form of these ODEs and new dynamical variables (such as $\Othree$) will depend on the particular NN, generically the system of coupled ODEs may not be closed at this order. Hence we continue generating new equations in the hierarchy by computing time derivatives of the new variables that appear. Altogether we obtain a hierarchy of coupled ODEs for dynamical variables $\O{s}(x_1,...,x_s,t)$ that involve particular types of contractions, over NN parameters, of high and low-order derivatives of the NN function with respect to parameters. We refer to this hierarchy of coupled ODEs as a function space description since it references dynamical variables whose arguments are all on sample space $x \in \mathbb{R}^{n_0}$ and the NN parameters are summed over in the description of the new variables.\footnote{These ODEs were first introduced and studied at a physics-level of rigor in \cite{DyerGurAri2020} for deep linear and ReLU networks, which we follow here, and then analyzed from a mathematically rigorous perspective in \cite{huangYau}. A related set of variables is introduced and studied in \cite{roberts2021principles}.}

Continuing the procedure described, we derive the evolution of $\Othree$ in terms of a new variable $\Ofour$,
\begin{equation}
    \frac{d \mathbb{O}_{3, t} (x,x',x_{\alpha})}{dt} = 
         - \sum_{\beta \in \mathcal{D}} R_{\beta, t} \mathbb{O}_{4,t}(x,x',x_{\alpha},x_{\beta}),
\end{equation}
and so on. To write a compact expression, we define 
\begin{equation}
\begin{split}
    & \O{1}(x_1) := f(x_1) \\
    & \O{2}(x_1,x_2) := \Theta(x_1,x_2) \\
    & \O{s}(x_1,...,x_s) := \sum_{\mu} \pderiv{\O{s-1}}{\theta_{\mu}} \pderiv{f(x_s)}{\theta_{\mu}} 
    ,\; s \geq 3,
\end{split}
\end{equation}
and they obey an associated hierarchy of coupled ODEs
\begin{equation}\label{ODE_hierarchy}
    \frac{d \mathbb{O}_{s,t} (x_1,...,x_s)}{dt} = - \sum_{\alpha \in \mathcal{D}} R_{\alpha, t} \mathbb{O}_{s+1,t}(x_1,...,x_s,x_{\alpha}).
\end{equation}

This system of equations has an appealing structure that is, at a high level, reminiscent of the BBGKY (Bogoliubov–Born–Green–Kirkwood–Yvon) hierarchy in statistical physics, where we might interpret $x_1, x_2, ...$ as interacting particles. While we will not pursue this correspondence further, we note that natural physical constraints can enable closure of the BBGKY hierarchy. Similarly, to make further progress we must find some natural means for closure of this system for deep NN dynamics. 

It turns out, as derived in \cite{DyerGurAri2020}, that the ``higher-order" (in $s$) variables $\mathbb{O}_s$ have a natural scale at initialization that is suppressed in inverse width (for deep NNs with specific choices of nonlinearities). Specifically, for a function $F_t(x)$ different from the $\mathbb{O}_{s}$ variables,
\begin{equation}\label{scaling_of_hierarchy}
    \E{\mathbb{O}_{s, t}(x_1,...,x_s)F_t(x)}{\theta_t} = 
    \begin{cases}
        \bigO(n^{-\frac{s-2}{2}}), & \; s \; \mathrm{even}\\
        \bigO(n^{-\frac{s-1}{2}}), & \; s \; \mathrm{odd}.
    \end{cases}
\end{equation}

In a randomly initialized NN in NTK parameterization, this scaling of expectation values can be derived by counting the number of sums and derivatives. (For deep linear networks, this would be a straightforward application of Wick's Theorem.) The scaling of expectation values holds during training as well, since the dynamical corrections to the $\mathbb{O}_s$ variables are governed by suppressed variables with larger $s$. If the training loss (tied to the contribution of $R_{\alpha, t}$ variables) decreases fast enough, the changes to the scaling of expectation values can be neglected compared to the scaling at initialization. (In particular, we know from the exactly solvable limit that in its vicinity, the training loss and hence $R_{\alpha}$ decrease exponentially in time, further suppressing the accumulation of corrections if $n$ is large.)

Therefore, we find that the contribution of higher-order variables $\mathbb{O}_{s}$ to the dynamics of the NN function $f_t(x)$ is suppressed in the coupled ODEs, and we can truncate the hierarchy at finite $s$ if the width $n$ is large and gradient flow is valid.

\subsubsection{Dynamics with leading order $1/n$ correction from finite width}

From the scaling of the expectation values in \eqref{scaling_of_hierarchy}, we have that $\Othree,\Ofour\sim\bigO(1/n)$, while $\O{s\geq 5}\sim \bigO(1/n^2)$. We aim to calculate finite-width NN dynamics correct to $1/n$ but dropping terms of higher order. As our focus is on the correction to \emph{dynamics} rather than the discrepancy between infinite and finite-width that exists already \emph{at initialization}, we will base our integration of the ODEs from a randomly initialized NN that is at large but \emph{finite} $n$. Hence, in this section $\mathbb{O}_{s, 0}$ variables (and in particular the NTK $\Theta_0$) refer to the initial values of these variables in a randomly initialized \emph{finite-width} network.

Now, since
\begin{equation}
    \frac{d \mathbb{O}_{4,t}(\cdot)}{dt} = - \sum_{\alpha \in \mathcal{D}} R_{\alpha, t} \mathbb{O}_{5, t}(\cdot, x_{\alpha}) \sim \bigO\left(\frac{1}{n^2}\right),
\end{equation}
\noindent based on the scaling of average values, we set the right-hand side to zero and take $\mathbb{O}_{4,t}(\cdot) = \mathbb{O}_{4, 0}(\cdot)$, i.e. equal to its initial value. (We use the symbol $\cdot$ here as substitute for the same set of arguments on both sides of the equation.) Examining next the preceding equation in the hierarchy, 

\begin{equation}
    \frac{d \mathbb{O}_{3,t}(\cdot)}{dt} \approx -\sum_{\alpha \in \mathcal{D}} 
    \underbrace{\bigg( f_t(x_{\alpha}) - y_{\alpha} \bigg)}_{ a_0 + a_1/n + a_2/n^2 + ....} 
    \underbrace{\mathcal{O}_{4, t}(\cdot, x_{\alpha})}_{b_1/n + b_2/n^2 + ...},
\end{equation}
we consider the variables on the right-hand side as having a power-series expansion in inverse width (with coefficients $\{a_i\}$ and $\{b_i\}$), with the expansion for the residual and for $\mathbb{O}_{4,t}$ beginning at $1/n^0$ and $1/n$, respectively. Hence, to compute $\mathbb{O}_{3, t}$ correct to $\mathcal{O}(1/n)$ we only need the $1/n^0$ contribution from the residual $R_{\alpha, t} = f_t(x_{\alpha}) - y_{\alpha}$, which is precisely the exactly solvable exponential-in-time dynamics we derived in Lecture 2 \eqref{eq:exactly_solvable_general_x} (albeit interpreting the quantities as originating from a randomly initialized, finite-width network). After substitution and integration, we obtain
\begin{equation}
    \mathbb{O}_{3,t}(\tilde{x}) = \mathbb{O}_{3,0}(\tilde{x}) 
    - \sum_{\alpha, \beta \in \mathcal{D}} \mathbb{O}_{4,0}(\tilde{x},x_{\alpha}) \, \left[ \Theta_0^{-1} \right]_{\alpha, \beta} \left( 1 - e^{-t \Theta_0} \right)_{\alpha, \beta} \bigg( f_0(x_{\beta}) - y_{\beta} \bigg),
\label{eq:O3correction}
\end{equation}
where we used the shorthand $\tilde{x}=(x_1,x_2,x_3)$ for three of the sample arguments and explicitly denote the matrix elements of $\Theta_0^{-1}$ and $e^{-t\Theta_0}$ that are needed. Note that the time-dependent term in \eqref{eq:O3correction} implicitly scales as $\sim 1/n$ due to this scaling in $\mathbb{O}_{4,0}$.


Our next step is to use the correction to $\mathbb{O}_{3,t}$ to correct the dynamical Neural Tangent Kernel. We write this as $\Theta_t = \Theta_0 + \Theta_t^{(1)} + \bigO(1/n^2)$, keeping in mind our overloaded notation so that $\Theta_0$ is extracted from a randomly initialized \emph{finite-width network}. Computing $\Theta_t^{(1)}$ is analytically tractable since it requires an integration against exponentials; to highlight the structure of the result, we perform it in the eigenbasis of $\Theta_0$, with eigenvalues $\{ \lambda_i \}$ and eigenvectors $\{ \hat{e}_i \}$:
\begin{equation}
\begin{split}
    \Theta^{(1)}_t(x_1,x_2) 
    \approx & - \underbrace{\Theta^{(1)}_0(x_1,x_2)}_{=0}
    - \int_0^t \mathrm{d}t' \, \sum_{\alpha \in \mathcal{D}}  
    \underbrace{\bigg( f_t(x_{\alpha}) - y_{\alpha}\bigg) }_{a_0 + a_1/n + ...} 
    \underbrace{\mathbb{O}_{3,t}(x_1,x_2,x_{\alpha})}_{c_1/n + c_2/n^2 + ...} \\
    = & - \int_0^t \, \mathrm{d}t' \sum_{i} \left(\mathbb{O}_{3,0} (\Vec{x}) \cdot \hat{e}_i \right) e^{-\lambda_i t'} 
    \left( R_0 \cdot \hat{e}_i  \right) \\
    & \, + \int_0^t \, \mathrm{d}t' \sum_{ij} \left(\hat{e}_i \cdot \mathbb{O}_{4,0}(\Vec{x}) \cdot \hat{e}_j \right)
    \cdot \frac{1}{\lambda_j} \left(1 - e^{-\lambda_j t'} \right) \left( R_0 \cdot \hat{e}_j  \right)
     e^{-t'\lambda_i} \left( R_0 \cdot \hat{e}_i  \right) \\
    = & - \sum_{i} \left(\mathbb{O}_{3,0}(\Vec{x}) \cdot \hat{e}_i \right) 
    \left( R_0 \cdot \hat{e}_i  \right)
    \left( \frac{1- e^{-\lambda_i t}}{\lambda_i} \right) \\
    & \, + \sum_{ij} \left(\hat{e}_i \cdot \mathbb{O}_{4,0} (\Vec{x}) \cdot \hat{e}_j \right)
    \frac{\left( R_0 \cdot \hat{e}_i  \right)\left( R_0 \cdot \hat{e}_j  \right)}{\lambda_j}
    \left[ \frac{1- e^{- \lambda_i t}}{\lambda_i} - \frac{1- e^{- (\lambda_i + \lambda_j) t}}{\lambda_i + \lambda_j} \right].
\end{split}    
\end{equation}

\noindent We have used the shorthand $\vec{x} = (x_1, x_2)$ for two of the sample arguments and defined the vector $R_0$ with elements $R_{\alpha, 0} = f_0(x_{\alpha}) - y_{\alpha}$; inner products with $\hat{e}_{i}, \hat{e}_{j}$ involve contracting the entries of these vectors with the sample degrees-of-freedom that are not explicitly referenced (e.g. $\mathbb{O}_{3,0}(\vec{x}) \cdot \hat{e}_{i} = \sum_{\alpha \in \mathcal{D}} \mathbb{O}_{3,0}(\vec{x}, x_{\alpha}) (\hat{e}_i)_{\alpha}$). 

Finally, we can use the correction to the Neural Tangent Kernel above to compute the correction to the function learned by the NN. For the NN function values evaluated on the training set $x_{\alpha} \in \mathcal{D}$, we obtain
\begin{equation}
    f_t(x_{\alpha}) = y_{\alpha} + \left[ e^{-\Theta_0 t} 
    \left( 1 - \int_0^t \mathrm{d}t' \, e^{-\Theta_0 t'} \Theta^{(1)}_{t'} e^{-\Theta_0 t'} \right) (f_0 - y) \right]_{\alpha} + \bigO(1/n^2),
\end{equation}

\noindent and we can similarly derive an expression for the function value $f_t(x)$ at an arbitrary point $x$.

\paragraph{Corrected dynamics at late times as $t\to \infty$} \mbox{} \\ 

For late times, these expressions predict exponential-in-time dynamics with an effective kernel $\Theta_0 + \Theta^{(1)}_{\infty}$,
\begin{equation}
    f(t) \to y + e^{-(\Theta_0 + \Theta^{(1)}_{\infty}) t} (f_0 - y),
\end{equation}
with the late-time correction
\begin{equation}
\begin{split}
    \Theta^{(1)}_{\infty} := & \lim_{t\to\infty} \Theta^{(1)}_t \\
    = &  - \sum_i \left(\mathbb{O}_{3,0}(\Vec{x}) \cdot \hat{e}_i \right) 
    \frac{R_0 \cdot \hat{e}_i}{\lambda_i} 
    + \sum_{ij} \frac{\left( R_0 \cdot \hat{e}_i  \right)\left( R_0 \cdot \hat{e}_j  \right)}{\lambda_i (\lambda_i + \lambda_j)}
    \left(\hat{e}_i \cdot \mathbb{O}_{4,0}(\Vec{x}) \cdot \hat{e}_j \right).
\end{split}
\end{equation}

\noindent Recall that these are $\mathcal{O}(1/n)$ corrections since both $\mathbb{O}_{3,0}, \mathbb{O}_{4,0} \sim 1/n$. While this is a valid theoretical description of feature learning (that is, the Neural Tangent Kernel changes from its initial value $\Theta_0$), we regard this is a regime of ``weak" feature learning since the change is small in comparison to the value of $\Theta_0$. Nonetheless, it is challenging to derive closed-form expressions for feature learning that maintain generality across architectures and datasets (the derivation above essentially is model and data agnostic, except for pathological settings), and it is intriguing to have such an expression for further analysis. 

\subsection{Large learning rate dynamics at large width: the ``catapult" mechanism}\label{subsec:nonperturbative}

The connection between infinite-width deep NNs, linear models, kernels, and GPs which was the subject of Lecture 2 holds up to a maximum learning rate $\eta_{\text{crit}}$ used in gradient descent. In fact, empirically one finds that a large but finite-width NN can be optimized to convergence at learning rates larger than this value \cite{lewkowycz2020large}. Is it possible to understand some aspects of this regime theoretically?

Indeed, consider a minimal NN model consisting of a single hidden-layer with no nonlinearities,
\begin{equation}
    f(x) = \frac{1}{\sqrt{n}} v^\T u x,
    \label{eq:more_general_uv}
\end{equation}
with parameters $v\in \mathbb{R}^n, \, u\in \mathbb{R}^{nx n_0}$, $m$ samples $(x_{\alpha}, y_{\alpha})$ with $x_{\alpha} \in \mathbb{R}^{n_0}, y_{\alpha} \in \mathbb{R}$, and trained with gradient descent on square loss in NTK parameterization. To illustrate the main features before returning to the more general case, we consider an even further simplified setting for this model: training on a single sample $(x,y) = (1,0)$ with $n_0=1$. We wish to understand the dynamics of 
\begin{equation}
    \mathcal{L}_t = \frac{f_t}{2} \quad \mathrm{with} \quad f_t = \frac{1}{\sqrt{n}} v^\T_t u_t .
\end{equation}
Gradient descent dynamics on the parameters is given by
\begin{align}
    u_{t+1} &= u_t - \frac{\eta}{\sqrt{n}} f_t \cdot v_t &
    v_{t+1} &= v_t - \frac{\eta}{\sqrt{n}} f_t \cdot u_t,
\end{align}
and the NTK is just a scalar, $\Theta_t(1,1) = \frac{1}{n} (\|u_t\|_2^2 + \|v_t\|_2^2) := \lambda_t$. Note that both $f_0, \Theta_0 \sim \bigO(1)$ at initialization. Instead of analyzing the dynamics in parameter space, we work in function space and -- in analogy with the construction of the hierarchy of coupled ODEs in Sec. \ref{sec:coupledODEhierarchy} -- write down an evolution for the function, Neural Tangent Kernel, and any other dynamical variables:
\begin{align}
    f_{t+1} &= f_t \left( 1- \eta \lambda_t + \frac{\eta^2 f_t^2}{n} \right) &
    \lambda_{t+1} &= \lambda_t  + \frac{\eta^2 f_t^2}{n} \left( \eta \lambda_t - 4 \right).
    \label{eq:catapult_model}
\end{align}

Surprisingly, for this simplified setting we can close (the discrete time version of) the hierarchy \eqref{ODE_hierarchy} exactly in terms of the variables $f_t$ and $\lambda_t$ alone. This is in contrast to more complex settings where a truncation scheme is required to close the system.

Let us analyze (\ref{eq:catapult_model}) in different regimes. In the $n\to\infty$ limit, we have
\begin{align}
    f_{t+1} =& f_t (1 - \eta \lambda_0) & \lambda_t = \lambda_0,
\end{align}
so that the NTK is constant and the function value (and hence loss) converges exponentially in time as long as $| 1-\eta \lambda_0 | < 1$. Consequently, for learning rates $\eta < \frac{2}{\lambda_0} := \eta_{\text{crit}}$, we obtain NTK dynamics. Backing off slightly from the limit while keeping $\eta <  \eta_{\text{crit}}$, we will obtain $\bigO(1/n)$ corrections to the dynamics, analogous to the perturbative corrections we investigated in \ref{sec:perturbative_NTK}. 

For learning rates $\eta > \frac{4}{\lambda_0}$, the last term in (\ref{eq:catapult_model}) is positive, causing $\lambda_t$ to increase with time and eventually diverge, along with the loss. In contrast, an interesting regime exists for $\frac{2}{\lambda_0} \leq \eta \leq \frac{4}{\lambda_0}$. Initially, the function and loss start to increase in magnitude, 
\begin{equation}
\begin{split}
    f_{t+1} &= f_t \overbrace{\left(1- \eta \lambda_t + \frac{\eta^2 f_t^2}{n} \right) }^{\geq 1 \; \mathrm{for}\; t=0} \\
    \lambda_{t+1} &= \lambda_t  + \frac{\eta^2 f_t^2}{n} \underbrace{\left( \eta \lambda_t - 4 \right)}_{< 0 \, \forall  t}.
\end{split}
\end{equation}

To see this, note that we can initially ignore the term $\eta^2 f^2_t /n $ in the dynamics of $f_t$, as $n$ is large, and since $| 1 - \eta \lambda_0 | > 1$, $| f_t |$ grows with time. However, once $|f_t| \sim \mathcal{O}(\sqrt{n})$, the second term in the dynamics of $\lambda_t$ yields $\mathcal{O}(1)$ contributions that enable $\lambda_t$ to decrease and dynamically adjust to the large learning rate. This in turn enables $|1 - \eta \lambda_t | < 1$ eventually and (combined with the $\eta^2$ term in the dynamics of $f_t$) results in the convergence of $f_t$ and the loss to a finite value. The mechanism at play here is that the local curvature (essentially captured by $\lambda_t$) adjusts dynamically to the larger learning rate, and optimization ``catapults" to a different region of the high-dimensional landscape in parameters $u, v$ than its initial condition. This catapult effect, enabling $|f_t| \sim \mathcal{O}(\sqrt{n})$, occurs on a time scale $t_{\ast} \sim \bigO(\log(n))$.

The transition between the ``NTK regime" and ``catapult regime" that occurs at $\eta_{\mathrm{crit}} = \frac{2}{\lambda_0}$ and becomes progressively sharper as $n \rightarrow \infty$ is reminiscent of a phase transition in dynamics. Indeed, there are measurable quantities that exhibit divergences near this transition. For example, the optimization time $t_{\epsilon}(\eta)$ needed to reach a loss of $\mathcal{O}(\epsilon)$ behaves as
\begin{equation}
    t_{\epsilon}(\eta) \sim \frac{1}{|\eta_{\mathrm{crit}} - \eta|},
\end{equation}

\noindent with exponent $\nu = 1$ (and dropped constants) in the vicinity of $\eta_{\mathrm{crit}}$, approached from below or above.

What is additionally surprising about this phenomenology is that, although we have studied a drastically simplified model, the catapult regime is empirically observed in a diverse range of realistic settings, including different datasets, NN architectures, and precise optimization choices (e.g. stochasticity in gradient descent and choice of standard vs. NTK parameterization) \cite{lewkowycz2020large}. To reiterate these empirical observations, one finds three regimes of dynamics in large width, deep NNs trained using stochastic gradient descent with learning rate $\eta$ and square loss:\footnote{$\lambda_0$ refers to the maximum eigenvalue of the Neural Tangent Kernel at initialization.}

\begin{enumerate}
    \item When $\eta \lesssim 2/\lambda_{0}$, the ``NTK" regime holds, namely the change in the dynamical Neural Tangent Kernel, $\Delta \Theta_t \overset{t\to \infty}{\longrightarrow} 0$, vanishes as $n$ gets larger. We can understand this regime with perturbative corrections discussed earlier in this lecture. The loss decreases fairly monotonically during optimization.
    \item When $ 2/\lambda_{0} \lesssim \eta \lesssim \eta_{\max}$, the dynamical Neural Tangent Kernel changes by a nonvanishing amount as $n \rightarrow \infty$, $\Delta \Theta_t \overset{t\to \infty}{\longrightarrow} \bigO(1)$, exhibiting a ``strong" form of feature learning. Here $\eta_{\max} = c/\lambda_{0}$, with $c$ being a $\mathcal{O}(1)$ constant. For the minimal model, $c = 4$; while this value is approximately observed in deep NNs with certain nonlinearities, in general $c$ is a non-universal constant. The loss behaves non-monotonically during optimization, with an initial increase early in training on a time scale $t \sim \mathcal{O}(\log(n))$. Optimization converges to a region with flatter curvature (as evidenced by the effect on the eigenvalues of the Neural Tangent Kernel).
    \item When $\eta \geq \eta_{\max}$, optimization diverges.
\end{enumerate}

Let us return to the model \eqref{eq:more_general_uv} with the more general setting of $m$ samples and dimensionality $n_0$ \cite{lewkowycz2020large}. Gradient descent on parameters takes the form

\begin{align}
u_{ia,  t+1} &= u_{ia, t} - \frac{\eta}{m \sqrt{n}} \sum_{\alpha \in \mathcal{D}} v_{i, t} x_{a \alpha} R_{\alpha, t} & v_{i, t+1} &= v_{i, t} - \frac{\eta}{m \sqrt{n}} \sum_{a, \alpha \in \mathcal{D}} u_{i a, t} x_{a \alpha} R_{\alpha, t},
\end{align}

\noindent where we use $R_{\alpha} = f_{\alpha} - y_{\alpha}$ as before. The Neural Tangent Kernel evaluated on the training data has matrix elements $\Theta_{\alpha \beta} = \frac{1}{n m} \bigg( |v|^2 x^T_{\alpha} x_{\beta} + x^T_{\alpha} u^T u x_{\beta} \bigg)$. Tracking the dynamics in the natural variables on function space (the residual $R_{\alpha}$ is directly related to $f_{\alpha}$) yields

\begin{equation}
\begin{split}
R_{\alpha, t+1} &= \sum_{\beta \in \mathcal{D}} (\delta_{\alpha \beta} - \eta \Theta_{\alpha \beta, t}) R_{\beta, t} + \frac{\eta^2}{n m} (x^T_{\alpha} \zeta_t) (f^T_t R_t) \\
\Theta_{\alpha \beta, t+1} &= \Theta_{\alpha \beta, t} - \frac{\eta}{nm} \left[ (x^T_{\beta} \zeta_t) f_{\alpha, t} + (x^T_{\alpha} \zeta_t) f_{\beta, t} + \frac{2}{m} (x^T_{\alpha} x_{\beta}) (R^T_t f_t) \right] + \\
& \frac{\eta^2}{n^2 m} \left[ |v_t|^2 (x^T_{\alpha} \zeta_t) (x^T_{\beta} \zeta_t) + (\zeta^T_t u^T_t u_t \zeta_t) (x^T_{\alpha} x_{\beta}) \right],
\label{eq:not_closed_eq}
\end{split}
\end{equation}

\noindent where have defined the vector $\zeta = \sum_{\alpha \in \mathcal{D}} R_{\alpha} x_{\alpha} /m \in \mathbb{R}^{n_0}$. This system of discrete time equations is not closed, unlike the version we considered in the simpler setting, and its closure does not arise naturally with the consideration of higher-order variables analogous to $\mathbb{O}_s$ for $s \geq 2$. However, we can approximately extract a two-variable closed system of equations that is reminiscent of the simpler system (\ref{eq:catapult_model}). Consider the dynamics of the Neural Tangent Kernel projected onto the residual,
\begin{align}
R^T_t \Theta_{t+1} R_t = R^T_t \Theta_{t} R_t + \frac{\eta}{n} \zeta^T_t \zeta_t \bigg( \eta R^T_t \Theta_t R_t - 4 f^T_t R_t \bigg).
\end{align}

\noindent Due to the form of the dominant term in the dynamics of the function (or residual), namely $\delta_{\alpha \beta} - \eta \Theta_{\alpha \beta_t}$, we might be inclined to approximate $R_t$ as becoming well-aligned with the maximum eigenvector of $\Theta$ at initialization, denoted $\hat{e}_{\text{max}}$. (Particularly in the catapult regime, the function and the residual grow exponentially fast, and this occurs along the $\hat{e}_{\text{max}}$ direction.) Hence, as a naive approximation we take $f_t \approx R_t \approx (\hat{e}_{\text{max}} \cdot R_t) \hat{e}_{\text{max}}$. This allows us to approximately simplify the equation for the projected kernel to an equation for the top NTK eigenvalue,
\begin{align}
    \lambda_{t+1} \approx \lambda_t + \frac{\eta}{n} \zeta^T_t \zeta_t (\eta \lambda_t - 4) 
\end{align}

\noindent which bears similarity to the simpler (\ref{eq:catapult_model}). Hence, we can understand how, despite the lack of closure in \eqref{eq:not_closed_eq}, it contains within it the mechanisms and universal phenomenology of (\ref{eq:catapult_model}), giving rise to distinct regimes of NN dynamics.


\newcommand{\ceil}[1]{\lceil #1 \rceil}
\newcommand{\twiddle}[1]{\widetilde{#1}}
\newcommand{\pr}[1]{\mathbb P\left(#1\right)}
\newcommand{\bk}[1]{\left \langle #1 \right \rangle}
\newcommand{\prend}{$\hfill \Box$}
\newcommand{\ls}{\leqslant}
\newcommand{\gr}{\geqslant}
\newcommand{\eps}{\varepsilon}
\renewcommand{\norm}[1]{\left|\left|#1\right|\right|}
\newcommand{\lr}[1]{\left(#1\right)}
\newcommand{\abs}[1]{\left|#1\right|}
\newcommand{\set}[1]{\left\{#1\right\}}
\renewcommand{\E}[1]{\mathbb E\left[#1\right]}
\newcommand{\inprod}[2]{\left \langle #1,#2\right\rangle }
\def\Var{\mathrm{Var}}  
\def\Cov{\mathrm{Cov}} 
\newcommand{\x}{\times}
\newcommand{\R}{\mathbb R}
\newcommand{\C}{\mathbb C}
\newcommand{\gives}{\rightarrow}
\newcommand{\mO}{\mathcal O}
\newcommand{\mA}{\mathcal A}
\newcommand{\mF}{\mathcal F}
\newcommand{\mN}{\mathcal N}
\newcommand{\mL}{\mathcal L}
\newcommand{\mS}{\mathcal S}
\newcommand{\mG}{\mathcal G}
\newcommand{\amO}{\overrightarrow{\mO}}
\newcommand{\mD}{\mathcal D}
\newcommand{\Z}{\mathbb Z}
\newcommand{\N}{\mathbb N}
\newcommand{\bkl}[1]{\bk{#1}_{(\ell)}}
\newcommand{\bkkl}[1]{\bk{#1}_{K^{(\ell)}}}
\newcommand{\bkal}[1]{\bk{#1}_{\kappa^{(\ell)}}}
\newcommand{\twomat}[4]{\lr{\begin{array}{cc} #1 & #2 \\ #3 & #4 \end{array}}}
\newcommand{\wkappa}{\widehat{\kappa}}
\newcommand{\kappal}{\kappa^{(\ell)}}
\newcommand{\Di}[1]{\frac{\partial}{\partial #1}}
\newcommand{\Disq}[1]{\frac{\partial^2}{\partial #1^2}}
\newcommand{\DDi}[2]{\frac{\partial^2}{\partial #1\partial #2}}
\newcommand{\aperp}{a_{\perp}}
\newcommand{\K}[3]{K_{(#1#2)}^{(#3)}}
\newcommand{\Kell}[2]{K_{(#1#2)}^{(\ell)}}
\newcommand{\kell}[2]{\kappa_{(#1)(#2)}^{(\ell)}}
\newcommand{\zz}{\sigma^2}
\newcommand{\oz}{\sigma\sigma'dz}
\newcommand{\oo}{(\sigma'dz)^2}
\newcommand{\tz}{(\sigma\sigma'' (dz)^2 + \sigma\sigma' d^2z)}
\newcommand{\tzs}{\sigma\sigma'' (dz)^2 + \sigma\sigma' d^2z}
\renewcommand{\to}{(\sigma'\sigma'' (dz)^3 + (\sigma')^2dz d^2z)}
\newcommand{\tos}{\sigma'\sigma'' (dz)^3 + (\sigma')^2dz d^2z}
\renewcommand{\tt}{(\sigma'' (dz)^2 + \sigma' d^2z)^2}
\newcommand{\mM}{\mathcal M}
\newcommand{\mI}{\mathcal I}
\newcommand{\mB}{\mathcal B}
\newcommand{\mP}{\mathcal P}
\newcommand{\sgn}[1]{\mathrm{sgn}\left(#1\right)}

\newtheorem{theorem}{Theorem}[section]
\newtheorem{lemma}[theorem]{Lemma}
\newtheorem{corollary}[theorem]{Corollary}
\newtheorem{proposition}[theorem]{Proposition}
\newtheorem{question}[theorem]{Question}
\newtheorem{conjecture}[theorem]{Conjecture}
\newtheorem{prob}{Problem}
\newtheorem{remark}[theorem]{Remark}
\newtheorem{assumption}[theorem]{Assumption}

\section{Lecture 4: Boris Hanin}\label{S:L4}

Lectures 4 and 5 are due to Boris Hanin.
They continue the trajectory of Yasaman Bahri's Lectures 1-3, focusing on asymptotic and perturbative calculations of the prior distribution of fully-connected neural networks.
Lecture 4 derives perturbative corrections to the NNGP.
Lecture 5 changes tack and discusses exact prior calculations specific to \textrm{ReLU} networks.

\subsection{Notation Dictionary}
From now on, there be a change of notation that we summarize here:
\begin{center}
\begin{tabular}{ |c|c| }
 \hline
 Lectures 1-3 & Lectures 4-5 \\
 \hline
 $\mathbb{E}\left[\cdot \right]$ & $\langle \cdot \rangle$ \\
 $0\le\ell<L$ & $1\le\ell<L+1$ \\
 $ z_i^{l=\ell} (x_\alpha)$ & $z_{i;\alpha}^{(\ell)}$ \\ 
 $K^{l = \ell}(x_\alpha, x_\beta)$ & $K^{(\ell)}_{\alpha \beta}$ \\ 
 $\phi( \cdot )$ (nonlinearity) & $\sigma(\cdot )$ \\ 
 $\sigma^2_b$, $\sigma^2_w$ & $C_b$, $C_w$\\
 \hline
\end{tabular}
\end{center}
\subsection{Notation}\label{S:notation}
Fix $L\geq 1$, $n_0,\ldots, n_{L+1}\geq 1$, and $\sigma:\R\gives \R$. We will consider a fully connected feed-forward network, which to an input $x_\alpha\in \R^{n_0}$ associates an output $z_\alpha^{(L+1)}\in \R^{n_{L+1}}$ as follows:
\begin{equation}\label{E:z-def}
    z_{i;\alpha}^{(\ell+1)}=\begin{cases} b_i^{(\ell+1)}+\sum_{j=1}^{n_\ell} W_{ij}^{(\ell+1)}\sigma\lr{z_{j;\alpha}^{(\ell)}},&\quad \ell \geq 1\\
    b_i^{(1)}+\sum_{j=1}^{n_0} W_{ij}^{(1)}x_{j;\alpha},&\quad \ell=0
    \end{cases}.
\end{equation}
We will have occasion to compute a variety of Gaussian integrals and will abbreviate
\[
\bk{f(z_\alpha)}_{K^{(\ell)}} = \int_{\R} f(z_\alpha) \exp\left[-\frac{z_\alpha^2}{2K_{\alpha\alpha}^{(\ell)}}  -\frac{1}{2}\log(2\pi K_{\alpha\alpha}^{(\ell)})\right] dz_\alpha 
\]
and more generally 
\[
\bk{f(z_\alpha,z_\beta)}_{K^{(\ell)}} = \int_{\R^2} f(z_\alpha,z_\beta) \exp\left[-\frac{1}{2}\sum_{\delta,\gamma\in \set{\alpha,\beta}} \lr{K^{(\ell)}}_{\gamma\delta}^{-1} z_\gamma z_\delta- \frac{1}{2}\log\det(2\pi K^{(\ell)})\right] dz_\alpha dz_\beta
\]
for Gaussian integrals in which $(z_\alpha,z_\beta)$ is a Gaussian vector with mean $0$ and covariance
\[
   K^{(\ell)} =\twomat{K_{\alpha\alpha}^{(\ell)}}{K_{\alpha\beta}^{(\ell)}}{K_{\alpha\beta}^{(\ell)}}{K_{\beta\beta}^{(\ell)}}.
\]

\subsection{Main Question: Statement, Answer, and Motivation}

\subsubsection{Precise Statement of Main Question}\label{S:question} Fix $L\geq 1,n_0,\ldots, n_{L+1}\geq 1,\sigma:\R\gives \R$ as well as constant $C_b\geq 0,C_W>0$. Suppose 
\begin{align}\label{E:init}
W_{ij}^{(\ell)} \sim \mN(0,C_W/n_{\ell-1}),\quad b_i^{(\ell)}\sim \mN(0,C_b)\qquad \text{independent}.
\end{align}
We seek to understand the distribution of the field
\[
x_\alpha\in \R^{n_0}~\mapsto ~z_\alpha^{(L+1)} \in \R^{n_{L+1}}
\]
when the hidden layer widths are large but finite:
\[
n_1,\ldots, n_L \simeq n\gg 1.
\]
\subsection{Answer to Main Question} \label{S:answer}
We will endeavor to show that the statistics of $z_\alpha^{(L+1)}$ are determined by
\begin{itemize}
    \item The universality class of the non-linearity $\sigma$ (determined by the large $\ell$ behavior of infinite width networks with this non-linearity). 
    \item The effective depth (or effective complexity)
    \[
    \frac{1}{n_1}+\cdots + \frac{1}{n_L}\simeq \frac{L}{n}.
    \]
\end{itemize}
Specifically, we'll see:
\begin{itemize}
    \item At init, $L/n$ measures both correlations between neurons and fluctuations in both values and gradients. (this lecture)
    \item $L/n$ measures the deviation from the NTK regime in the sense that the change in the NTK from one step of GD scales like $L/n$. Thus, the (frozen) NTK regime corresponds to the setting in which the effective depth $L/n$ tends to $0$. Moreover, the extent of feature learning, in the sense of figuring out how much the network Jacobian changes at the start of training, is measured by $L/n$. (next lecture)
    \item $L/n$ measures the extent of feature learning in the sense that the entire network function at the end of training scales like the NTK answer plus $L/n$ plus errors of size $(L/n)^2$ (see Chapter $\infty$ in \cite{roberts2021principles}).
\end{itemize}
This suggests an interesting phase diagram (see Figure \ref{fig:phase}).
\begin{figure}[h]
    \centering
    \includegraphics[scale=.8]{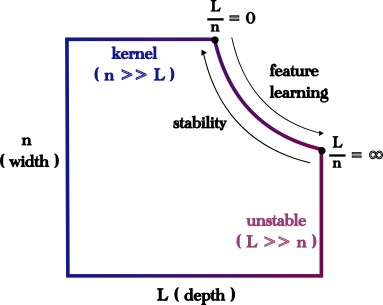}
    \caption{Partial Phase Diagram for Fully Connected Networks with NTK Initialization}
    \label{fig:phase}
\end{figure}

\subsection{Motivations} Before attempting to make precise our answer in \S \ref{S:answer}, we give several motivations for studying our main question: 
\begin{enumerate}
    \item Our first motivation is ML-centric. Namely, to use a neural network in practice requires choosing many hyperparameters, including
    \begin{itemize}
        \item width $n$
        \item depth $L$
        \item non-linearity $\sigma$
        \item initialization variances $C_b,C_W$
        \item learning rates $\eta$
        \item batch sizes $\abs{\mB}$
        \item ($\ell_1$ or $\ell_2$) regularization strength
    \end{itemize}
    Doing direct hyperparameter search is very costly. By studying random networks, we can understand in which combinations these hyperparameters appear in the distribution of $z_\alpha^{(L+1)}$ and, in particular, how to choose them in a coordinated manner so that $z_\alpha^{(L+1)}$ is non-degenerate (say near the start of training) at large values of $L,n$ and training steps.
    \item Our second motivation is mathematical/theoretical. Namely, random fully connected neural networks are non-linear generalizations of random matrix products. Indeed by taking $n_\ell \equiv n$, $\sigma(t)=t$, $C_b=0,C_W=1$, we see that
    \[
    z_\alpha^{(L+1)} = W^{(L+1)}\cdots W^{(1)}x_\alpha
    \]
    is simply a linear statistic of product of $L+1$ iid random matrices. Products of random matrices appear all over the place. When $L=1$ (or more generally $L$ is fixed and finite) and $n\gives \infty$, this is like Wigner's (or Wishart's) random theory. In contrast, when $n$ is fixed and $L\gives \infty$, this is the study of the long time behavior of random dynamical systems. This is the world of the multiplicative ergodic theorem and is used for example in studying Anderson localization in $1d$. A key point is that these two regimes are very different and what happens when both $n,L$ are large is relatively poorly understood, even for this random matrix model. 
    \item The final motivation is again ML-centric. As Yasaman showed in her lectures, when $L$ is fixed and $n\gives \infty$, fully connected networks with the initialization \eqref{E:init} are in the (frozen) NTK regime. In this setting, the entire training dynamics (at least on MSE with vanishingly small learning rates) are determined by the behavior at initialization. Thus, it is the properties of neural networks at init that allow us to describe the generalization behavior and training dynamics. In particular, by doing perturbation theory \textit{directly for the end of training}, it is possible to understand (see Chapter $\infty$ of \cite{roberts2021principles}) training in the near-NTK regime in which the NTK changes to order $1/n$ (really $L/n$). 
\end{enumerate}

\subsection{Intuition for Appearance of $L/n$} Before proceeding to explain how to compute finite width corrections to the statistics of random neural networks, we pause to elaborate a simple intuition for why it is $L/n$, rather than some other combination of $L$ and $n$, that should appear. For this, let us consider the very simple case of random matrix products 
\[
n_\ell\equiv n,\, \sigma(t)=t,\, C_b=0,C_W=1
\]
so that
\[
z_\alpha^{(L+1)}(x) = W^{(L+1)}\cdots W^{(1)}x_\alpha,\qquad W_{ij}^{(\ell)}\sim \mN(0,1/n)\,\, iid.
\]
Assuming for convenience that $\norm{x_\alpha}$ is bounded, let's try to understand what is perhaps the simplest random variable
\[
X_{n,L+1}:=\norm{z_{\alpha}^{(L+1)}}
\]
associated to our matrix product. In order to understand its distribution recall that for any $k\geq 1$ a chi-squared random variable with $k$ degrees of freedom is given by 
\[
\chi_{k}^2 :\stackrel{d}{=} \sum_{j=1}^k X_j^2,\qquad X_j\sim \mN(0,1)\,\,iid.
\]
Recall also that for any unit vector $u$ we have that if $W\in \R^{n\times n}$ is a matrix with $W_{ij}\sim \mN(0,1/n)$ then
\[
Wu \stackrel{d}{=} \mN(0,\frac{1}{n}\mathrm{I}_n),\qquad \norm{Wu}^2\stackrel{d}{=} \frac{1}{n}\chi_n^2,\qquad \norm{Wu}\perp \frac{Wu}{\norm{Wu}},
\]
where $\perp$ denotes conditional independence.
To use this let's write
\begin{align*}
    X_{n,L+1}&=\norm{ W^{(L+1)}\cdots W^{(1)}x_\alpha}=\norm{ W^{(L+1)}\cdots W^{(2)}\frac{W^{(1)}x_\alpha}{\norm{W^{(1)}x_\alpha}}}\norm{W^{(1)}x_\alpha}.
\end{align*}
Note that 
\[
\lr{\frac{1}{n}\chi_n^2}^{1/2}\stackrel{d}{=}\norm{W^{(1)}x_\alpha} \perp \frac{W^{(1)}x_\alpha}{\norm{W^{(1)}x_\alpha}}\in S^{n-1}.
\]
Thus, in fact the presentation above allows us to write $X_{n,L+1}$ as a product of two independent terms! Proceeding in this way, we obtain the following equality in distribution:
\[
X_{n,L+1} \stackrel{d}{=} \exp\left[\sum_{\ell=1}^{L+1} Y_\ell\right] ,\qquad Y_\ell \sim \frac{1}{2}\log \lr{\frac{1}{n}\chi_n^2}\,\, iid.
\]
\noindent \textbf{Exercise.} Show that 
\[
\E{\frac{1}{2}\log \lr{\frac{1}{n}\chi_n^2}}=-\frac{1}{4n} + O(n^{-2}),\qquad \Var\left[\frac{1}{2}\log \lr{\frac{1}{n}\chi_n^2}\right] = \frac{1}{4n}+O(n^{-2}).
\]\\
Thus, we see that 
\[
X_{n,L+1}\stackrel{L\gg 1}{\approx} \exp\left[\mathcal N\lr{-\frac{L}{4n}, \frac{L}{4n}}\right]
\]
and that \textbf{taking $n$ large in each layer tries to make each $Y_j$ close to $1$ but with errors of size $1/n$. When we have $L$ such errors, the total size of the error is on the order of $L/n$.}

\subsection{Summary of Yasaman's Lectures 1 - 3}\label{S:yasaman}
We summarize part of Yasaman’s lectures in one long theorem. For this, recall that a free (i.e. Gaussian) field is one in which the joint distribution of the field at any finite number of points is Gaussian. Hence, free fields are completely determined by their one and two-point functions. 
\begin{theorem}[GP + NTK Regime for Networks at Fixed Depth and Infinite Width]\label{T:yasaman}
Fix $L,n_0,n_{L+1},\sigma$. Suppose that at the start of training we initialize as in \eqref{E:init}.
\begin{itemize}
    \item[(i)] \textbf{GP at Init.} As $n_1,\ldots, n_L\gives \infty$, the field $x\mapsto z^{(L+1)}(x)$ converges weakly in distribution to a free (Gaussian) field with a vanishing one point function
    \begin{align*}
        \lim_{n_1,\ldots, n_L\gives \infty}\E{z_{i;\alpha}^{(L+1)}} &= 0
    \end{align*}
    and a two point function that factorizes across neurons
    \begin{align*}
         \lim_{n_1,\ldots, n_L\gives \infty}\Cov\lr{z_{i;\alpha}^{(L+1)},z_{j;\beta}^{(L+1)}} &= \delta_{ij}K_{\alpha\beta}^{(L+1)}.
    \end{align*}
    Moreover, the two point function is given by the following recursion
    \begin{equation}\label{E:K-rec}
    K_{\alpha\beta}^{(\ell+1)} = \begin{cases} C_b + C_W  \bk{\sigma(z_\alpha)\sigma(z_\beta)}_{K^{(\ell)}},  &\quad \ell \geq 1\\
C_b + \frac{C_W}{n_0}x_\alpha \cdot x_\beta,&\quad \ell = 0
\end{cases},
    \end{equation}
If $C_b,C_W$ are chosen by ``tuning to criticality'' (e.g. $C_b=0, C_W=2$ for ReLU or $C_b=0,C_W=1$ for $\tanh$) in the sense that 
\begin{align*}
&\exists K_*\geq 0\quad \text{s.t.}\quad K_* = C_b+C_W\bk{\sigma^2}_{K_*}\\
&\frac{\partial K_{\alpha\alpha}^{(\ell+1)}}{\partial K_{\alpha\alpha}^{(\ell)}}\bigg|_{K_{\alpha\alpha}^{(\ell)}=K_*} = \chi_{||;\alpha}^{(\ell)}=\frac{C_W}{2}\bk{\partial^2 \sigma^2}_{K_*}=1\\
&\frac{\partial K_{\alpha\beta}^{(\ell+1)}}{\partial K_{\alpha\beta}^{(\ell)}}\bigg|_{K_{\alpha\alpha}^{(\ell)}= {K_{\beta\beta}^{(\ell)} = {K_{\alpha\beta}^{(\ell)} =K_*}}} =\chi_{\perp}^{(\ell)}=C_W\bk{(\sigma')^2}_{K_*} =1,
\end{align*}
then 
    \[
    K_{\alpha\alpha}^{(\ell)} \simeq \ell^{-\delta_1},\qquad \delta_1\in [0,1]
    \]
    and 
    \begin{equation}\label{E:corr-prop}
        \mathrm{Corr}_{\alpha\beta}^{(\ell)}:=\frac{K_{\alpha\beta}^{(\ell)}}{\lr{K_{\alpha\alpha}^{(\ell)}K_{\beta\beta}^{(\ell)}}^{1/2}} \simeq 1- C_\sigma \ell^{-\delta_2},\qquad \delta_2\in [1,2].
    \end{equation}

    \item \textbf{Equivalence to Linear Model in Small LR Optimization with MSE.} If $\theta =\set{W^{(\ell)},b^{(\ell)}} $ is initialized to be $\theta_0$ as in \eqref{E:init} and is optimized by gradient flow (or GD with learning rate like $n^{-1/2}$) on empirical mean squared error over a fixed dataset, then as $n_1,\ldots, n_L\gives \infty$ optimization is equivalent to first linearizing
    \[
    z_\alpha^{(L+1)}(\theta)\quad \mapsto \quad z_\alpha^{(L+1)}(\theta_0) + \nabla_\theta z_\alpha^{(L+1)}(\theta_0)\lr{\theta - \theta_0}
    \]
    and then performing gradient flow on the same loss. The corresponding (neural tangent) kernel
    \[
    \Theta_{\alpha\beta}^{(L+1)}:= \nabla_\theta z_\alpha^{(L+1)}(\theta_0)^T \nabla_\theta z_\beta^{(L+1)}(\theta_0) \in \R^{n_{L+1}\times n_{L+1}}
    \]
    satisfies a recursion similar to \eqref{E:K-rec}.
\end{itemize}
\end{theorem}

\subsection{Formalizing Inter-Neuron Correlations and Non-Gaussian Fluctuations}\label{S:results}
To formulate our main result for this lecture define the normalized connected 4 point function:
\[
\kappa_{4;\alpha}^{(\ell)}  =\frac{1}{3} \kappa\lr{z_{i;\alpha}^{(\ell)},z_{i;\alpha}^{(\ell)},z_{i;\alpha}^{(\ell)},z_{i;\alpha}^{(\ell)}}=\frac{1}{3}\lr{\E{\lr{z_{i;\alpha}^{(\ell)}}^4}- 3 \E{\lr{z_{i;\alpha}^{(\ell)}}^2}^2}.
\]
Note that $\kappa_{4;\alpha}^{(\ell)}$ captures both fluctuations 
\[
\Var\left[\lr{z_{i;\alpha}^{(\ell)}}^2\right] = 3\kappa_{4;\alpha}^{(\ell)} + 2\E{\lr{z_{i;\alpha}^{(\ell)}}^2}^2
\]
and non-Gaussianity (in the sense that if $z_{i;\alpha}^{(\ell)}$ is Gaussian, then $\kappa_{4;\alpha}^{(\ell)}=0$). \\

\noindent \textbf{Exercise.} Show that
\[
\kappa_{4;\alpha}^{(\ell)}:=\Cov\lr{\lr{z_{i;\alpha}^{(\ell)}}^2, \lr{z_{j;\alpha}^{(\ell)}}^2},
\]
allowing us to interpret $\kappa_{4;\alpha}^{(\ell)}$ as a measure of inter-neuron correlations.\\

Since as $n_1,\cdots, n_L\gives\infty$, neurons are independent and Gaussian, we have that
\[
\lim_{n_1,\ldots, n_{\ell-1}\gives \infty} \kappa_{4;\alpha}^{(\ell)} = 0.
\]
Our main purpose in this lecture is obtain the following characterization of $\kappa_{4;\alpha}^{(\ell)}$. 

\begin{theorem}\label{T:main-L4}
Fix $L, n_0, n_{L+1},\sigma.$ Suppose that the weights and biases are chosen as in \eqref{E:init} and that 
\[
n_1,\ldots, n_L \simeq n \gg 1.
\]
The four point function is of order $O(n^{-1})$ and satisfies the following recursion:
\[
\kappa_4^{(\ell+1)}=\frac{C_W^2}{n_\ell} \Var_{K^{(\ell)}}\left[\sigma^2\right]+\lr{\chi_{||;\alpha}^{(\ell)}}^2\kappa_4^{(\ell)} + O(n^{-2}).
\]
Thus, at criticality and uniform width ($n_\ell=n$), we have
\[
\frac{\kappa_{4;\alpha}^{(L+1)}}{\lr{K_{\alpha\alpha}^{(L+1)}}^2} = C_\sigma \frac{L}{n} + O_{L,\sigma}(n^{-2}).
\]
Moreover, for any fix $m\geq 1$ and any ``reasonable’’ function $f:\R^m\gives \R$ we may write
\begin{align*}
\E{f\lr{z_{i;\alpha}^{(\ell)}, \, i =1,\ldots, m}} &= \bk{ f\lr{z_{i;\alpha}, \, i =1,\ldots, m}}_{G^{(\ell)}}\\
&+\frac{\kappa_4^{(\ell+1)}}{8}\bk{\bigg(\sum_{j=1}^m \partial_{z_{j;\alpha}}^4 + \sum_{\substack{j_1,j_2=1\\ j_1\neq j_2}}^{m} \partial_{z_{j_1;\alpha}}^2\partial_{z_{j_2;\alpha}}^2\bigg) f\lr{z_{i;\alpha}, \, i =1,\ldots, m}}_{K^{(\ell)}}\\
&+O(n^{-2}).
\end{align*}
Here, $G^{(\ell)}$ is the dressed two point function
\[
G_{\alpha\beta}^{(\ell)}=\E{z_{i;\alpha}^{(\ell)}z_{i;\beta}^{(\ell)}} = K_{\alpha\beta}^{(\ell)}+O(n^{-1}).
\]
\end{theorem}
This Theorem is originally derived in a physics way in the breakthrough paper of Yaida \cite{yaida2019non}. It was then rederived, again at a physics level of rigor in Chapter $4$ of \cite{roberts2021principles}. Finally, it was derived in a somewhat different, and more mathematical, way in \cite{hanin2022correlation}.

\subsection{Proof of Theorem \ref{T:main-L4}}

\subsubsection{A Bit of Background}
To study a general non-Gaussian random vector $z=(z_1,\ldots, z_m)$, we will understand its characteristic function
\[
\widehat{p}_z(\xi):=\E{e^{-iz\cdot \xi}} = \int_{\R^m} e^{-i\sum_{j=1}^m z_j\xi_j} p(z)dz.
\]
Its utility is:
\begin{enumerate}
    \item For any reasonable $f$ we can write the expectation of $f(z)$ using the characteristic function by taking a Fourier transform: 
    \[
    \E{f(z)} = \int_{\R^m} f(z)p(z)dz = \int_{\R^m} \widehat{f}(\xi)\widehat{p}_z(\xi)d\xi
    \]
    \item A Gaussian with $0$ variance $K$ has the simplest characteristic function:
    \[
    z\sim \mN(0,K)\quad \Rightarrow \quad \widehat{p}_z(\xi_1,\ldots, \xi_m) = \exp\left[-\frac{1}{2}\sum_{j_1,j_2=1}^m K_{j_1j_2}\xi_{j_1} \xi_{j_2}\right].
    \]
\item Multiplication of $\widehat{f}(\xi)$ by $\xi$ corresponds to differentiation:
\[
\xi_j\widehat{f}(\xi) = \widehat{-i\partial_j f}(\xi).
\]
\end{enumerate}
We will need the following 
\begin{proposition}\label{P:Gaussian-marginals}
Let $W = \lr{W_1,\ldots, W_n}\sim N(\mu, \Sigma)$. Then, for any independent (e.g. constant) matrix $A\in \R^{k\times n}$, we have
\[
AW\sim \mN(A\mu, A\Sigma A^T).
\]
\end{proposition}

\subsubsection{First Step: Reduce to Collective Observables}
\textbf{Definition.} For any $f:\R^m\gives \R$ we will say that 
\[
\mO_f^{(\ell)} = \frac{1}{n_\ell}\sum_{j=1}^{n_\ell} f(z_{j;\alpha}^{(\ell)},\,\alpha=1,\ldots,m)
\]
is a \textit{collective observable}.\\

\noindent Collective observables play a crucial role in our analysis. Indeed, our first step is to rewrite all the quantities in Theorem \ref{T:main-L4} in terms of such objects. Let us fix any $m\geq 1$. We have
\[
\E{f(z_{i;\alpha}^{(\ell+1)},\, i=1,\ldots, m)} = \int_{\R^m} \widehat{f}(\xi) \E{\exp\left[-i\sum_{j=1}^m \xi_j z_{j;\alpha}^{(\ell+1)}\right]} d\xi. 
\]
We begin by applying Proposition \ref{P:Gaussian-marginals} to  simplify the characteristic function of $(z_{i;\alpha}^{(\ell+1)},\, i=1,\ldots, m)$. 

\begin{lemma}\label{L:cond-cov}

Conditional on $z_\alpha^{(\ell)}$, 
\[
\lr{z_{\alpha}^{(\ell+1)}}_{i=1}^{n_{\ell+1}} \text{ is a Gaussian with mean }0\text{ and covariance } \Sigma_{\alpha}^{(\ell)}\cdot \mathrm{I},
\]
where
\begin{align*}
\Sigma_{\alpha}^{(\ell)}=C_b + \frac{C_W}{n_\ell}\sum_{j=1}^{n_\ell}\sigma\lr{z_{j;\alpha}^{(\ell)}}^2
\end{align*}
is a collective observable. In particular, for each $\xi=\lr{\xi_1,\ldots, \xi_m}$, we have
\begin{align*}
\E{e^{-i\sum_{i=1}^{m} \xi_i z_{i;\alpha}^{(\ell+1)}}}
&=\E{e^{-\frac{1}{2}\norm{\xi}^2 \Sigma_{\alpha}^{(\ell)}}}
\end{align*}
and, moreover, 
\begin{align*}
    \kappa_{4}^{(\ell+1)} = \Var\left[\Sigma_\alpha^{(\ell)}\right].
\end{align*}
\end{lemma}

\begin{proof}
We have
\[
z_{i;\alpha}^{(\ell+1)} = (\sigma(z_\alpha^{(\ell)})~ 1)(W^{(\ell+1)}~ b^{(\ell+1)})^T.
\]
Thus, if we are given $\sigma(z_\alpha^{(\ell}) $, then $\lr{z_{i;\alpha}^{(\ell+1)}}_{i=1}^{n_{\ell+1}}$ are iid Gaussian with mean $0$ and
\begin{align*}
\Sigma_{\alpha}^{(\ell)} &= \E{\lr{z_{i;\alpha}^{(\ell+1)}}^2}=C_b + \frac{C_W}{n_\ell}\sum_{j=1}^{n_\ell}\lr{\sigma\lr{z_{j;\alpha}^{(\ell)}}}^2.
\end{align*}
In particular, 
\begin{align*}
\E{e^{-i\sum_{j=1}^{m} \xi_j z_{j;\alpha}^{(\ell+1)}}} &= \E{\E{e^{-i\sum_{j=1}^{m} \xi_j z_{ij\alpha}^{(\ell+1)}} ~\bigg|~z_{\alpha}^{(\ell)}} }=\E{e^{-\frac{1}{2}\norm{\xi}^2 \Sigma_{\alpha}^{(\ell)}}}.
\end{align*}
\end{proof}

\noindent We have therefore found that for any reasonable $f$,
\[
\E{f(z_{i;\alpha}^{(\ell+1)},\,i=1,\ldots, m)}=\int_{\R^{m}} \widehat{f}\lr{\xi} \E{e^{-\frac{1}{2}\norm{\xi}^2 \Sigma_{\alpha}^{(\ell)}}}d\xi. 
\]

\subsection{Step 2: Decompose the Self-Averaging Observable $\Sigma_{\alpha}^{(\ell)}$ into a Mean and Fluctuation}
Since $\Sigma_{\alpha}^{(\ell)}$ is a collective observable, it makes sense to consider
\[
G_\alpha^{(\ell)}:=\E{\Sigma_\alpha^{(\ell)}},\qquad  \Delta_{\alpha}^{(\ell)}:=\Sigma_\alpha^{(\ell)}-\E{\Sigma_\alpha^{(\ell)}}.
\]
The scalar $G_\alpha^{(\ell)}$ is sometimes referred to as a dressed two point function. \\

\noindent \textbf{Exercise.} Show for any observables of the form
\[
\mO_f^{(\ell)} = \frac{1}{n_\ell}\sum_{j=1}^{n_\ell} f(z_{j;\alpha}^{(\ell)})
\]
that 
\[
\E{\prod_{j=1}^q \lr{\mO_{f_j}^{(\ell)} - \E{\mO_{f_j}^{(\ell)}}}} = O_{q,f_j,\ell, \sigma}\lr{n^{-\lceil \frac{q}{2}\rceil }}.
\]
Hint: do this in several steps:
\begin{enumerate}
\item[(a)] First check this when $\ell = 1$. This is easy because neurons are independent in the first layer. 
\item[(b)] Next assume that $\sigma$ is a polynomial and show that if you already know that the result holds at layer $\ell$, then it must also hold at layer $\ell+1$. This is not too bad but requires some book-keeping.
\item[(c)] Show that the full problem reduces to the case of polynomial activations. This is somewhat tricky. 
\end{enumerate}

\subsection{Step 3: Expand in Powers of Centered Collective Observables}
We have
\begin{align*}
\E{f(z_{i;\alpha}^{(\ell)},\,i=1,\ldots, m)}&=\int_{\R^{n_{m}}} \widehat{f}\lr{\xi} \E{e^{-\frac{1}{2}\norm{\xi}^2 \Sigma_{\alpha}^{(\ell)}}}d\xi\\
&= \int_{\R^{n_{m}}} \widehat{f}\lr{\xi} e^{-\frac{1}{2}\norm{\xi}^2 G_\alpha^{(\ell)}} \E{e^{-\frac{1}{2}\norm{\xi}^2 \Delta_{\alpha}^{(\ell)}}} d\xi.    
\end{align*}
Applying the exercise above we may actually Taylor expand to find a power series expansion in $1/n$: 
\begin{align*}
\E{e^{-\frac{1}{2}\norm{\xi}^2 \Delta_{\alpha}^{(\ell)}}} &=\sum_{q\geq 0} \frac{(-1)^q}{2^qq!} \norm{\xi}^{2q} \E{\lr{\Delta_{\alpha}^{(\ell)}}^q}= 1 +\frac{1}{8}\norm{\xi}^4 \E{\lr{\Delta_{\alpha}^{(\ell)}}^2} + O(n^{-2}).
\end{align*}
Putting this all together yields 
\[
\E{f(z_{i;\alpha}^{(\ell)},\,i=1,\ldots, m)}=\int_{\R^{n_{m}}}\lr{ 1 +\frac{1}{8}\norm{\xi}^4 \E{\lr{\Delta_{\alpha}^{(\ell)}}^2}}\widehat{f}\lr{\xi} e^{-\frac{1}{2}\norm{\xi}^2 G_\alpha^{(\ell)}} d\xi + O(n^{-2}).
\]
In particular, we obtain 
\[
\E{f(z_{i;\alpha}^{(\ell)},\,i=1,\ldots, m)}=\bk{f}_{G_\alpha^{(\ell)}} + \frac{1}{8} \E{\lr{\Delta_{\alpha}^{(\ell)}}^2} \bk{\lr{\sum_{j=1}^m \partial_{z_{j;\alpha}^{(\ell)}}^2}^2 f}_{G_{\alpha}^{(\ell)}} + O(n^{-2}).
\]
\noindent \textbf{Exercise.} Show that
\[
\bk{f}_{G_\alpha^{(\ell)}}=\bk{f}_{K_{\alpha\alpha}^{(\ell)}} + O(n^{-1}).
\]
Hint: define
\[
S_\alpha^{(\ell)}:=G_\alpha^{(\ell)} - K_{\alpha\alpha}^{(\ell)}.
\]
 We already know that $\E{\lr{\Delta_\alpha^{(\ell)}}^2}=O(n^{-1})$. Now obtain a recursion for $S_\alpha^{(\ell)}$ using the perturbative expansion above and check that the solution is of order $O(n^{-1})$.

\subsection{Step 4: Relating $k_{4;\alpha}^{(\ell+1)} $ to the Dressed 2 Point Function and Obtaining Its Recursion}
Recall that Lemma \ref{L:cond-cov} we saw that
\[
k_{4;\alpha}^{(\ell+1)} = \E{\lr{\Delta_{\alpha}^{(\ell)}}^2}. 
\]
Moreover,
\begin{align*}
\E{\lr{\Delta_{\alpha}^{(\ell)}}^2} = \frac{1}{n_\ell}\E{\lr{X_{1;\alpha}^{(\ell)}}^2} + \lr{1-\frac{1}{n_\ell}}\E{X_{1;\alpha}^{(\ell)}X_{2;\alpha}^{(\ell)}},
\end{align*}
where
\[
X_{j;\alpha}^{(\ell)}:=C_W\lr{\sigma(z_{j;\alpha}^{(\ell)})^2 - \E{\sigma(z_{j;\alpha}^{(\ell)})^2 }}. 
\]
Applying the result of Step 3 (and a previous exercise) yields
\begin{align*}
\frac{1}{n_\ell}\E{\lr{X_{1;\alpha}^{(\ell)}}^2}  = \frac{1}{n_\ell} C_W^2\bk{\lr{\sigma^2-\bk{\sigma^2}_{K_{\alpha\alpha}^{(\ell)}}}^2}_{K_{\alpha\alpha}^{(\ell)}}+O(n^{-2}) = \frac{C_W^2}{n_\ell}\Var_{K^{(\ell)}}[\sigma^2]+O(n^{-2}).
\end{align*}
Finally, note that
\[
0 = \E{X_{i;\alpha}^{(\ell)}}= \bk{X_{i;\alpha}^{(\ell)}}_{G_{\alpha}^{(\ell)}}+O(n^{-2}). 
\]
Hence, 
\begin{align*}
    \E{X_{1;\alpha}^{(\ell)}X_{2;\alpha}^{(\ell)}}& = \kappa_{4;\alpha}^{(\ell)}\bk{\lr{\frac{1}{8}\sum_{j=1}^2 \partial_{z_{j;\alpha}}^4 +\frac{1}{4}\partial_{z_{1;\alpha}}^2\partial_{z_{2;\alpha}}^2}X_{1;\alpha}^{(\ell)}X_{2;\alpha}^{(\ell)}}_{K_{\alpha\alpha}^{(\ell)}}+O(n^{-2})\\
    &=\lr{\frac{C_W}{2}\bk{\partial^2\sigma^2}_{K_{\alpha\alpha}^{(\ell)}}}^2\kappa_{4;\alpha}^{(\ell)} + O(n^{-2})\\
    &=\lr{\chi_{||;\alpha}^{(\ell)}}^2 \kappa_{4;\alpha}^{(\ell)} + O(n^{-2}).
\end{align*}
\subsection{Step 5: Solving the $4$ point function recursion}
In this section, we solve the four point function recursion 
\[
\kappa_{4}^{(\ell+1)}=\frac{C_W^2}{n_\ell}\Var_{K^{(\ell)}}[\sigma^2]+ \lr{\chi_{||;\alpha}^{(\ell)}}^2 \kappa_{4;\alpha}^{(\ell)} + O(n^{-2})
\]
in the special case when 
\[
\sigma(t)=\mathrm{ReLU}(t)=t {\bf 1}_{t>0}.
\]
First of all, as Yasaman showed, we have
\begin{align*}
K_{\alpha\alpha}^{(\ell+1)}=C_b+C_W\bk{\sigma^2(z_\alpha)}_{K_{\alpha\alpha}^{(\ell)}} &= C_b+\frac{C_W}{2}K_{\alpha\alpha}^{(\ell)}\\
\chi_{||;\alpha}^{(\ell+1)}= \frac{\partial K_{\alpha\alpha}^{(\ell+1)}}{\partial K_{\alpha\alpha}^{(\ell)}} = \frac{C_W}{2}.
\end{align*}
So we are at criticality only if
\[
C_b=0,\qquad C_W=2.
\]
With this, we have
\[
\chi_{||;\alpha}^{(\ell)}\equiv 1,\qquad K_{\alpha\alpha}^{(\ell)} \equiv \frac{2}{n_0}\norm{x_\alpha}^2. 
\]
Thus, 
\begin{align*}
    \Var_{K^{(\ell)}}[\sigma^2] &= \bk{\sigma^4}_{K_{\alpha\alpha}^{(\ell)}}-\lr{\bk{\sigma^2}_{K_{\alpha\alpha}^{(\ell)}}}^2=\frac{3}{2}\lr{K_{\alpha\alpha}^{(\ell)}}^2 - \frac{1}{4}\lr{K_{\alpha\alpha}^{(\ell)}}^2=\frac{5}{4}\lr{K_{\alpha\alpha}^{(\ell)}}^2.
\end{align*}
So we find
\[
\frac{\kappa_{4}^{(\ell)}}{\lr{K_{\alpha\alpha}^{(\ell)}}^2}=\sum_{\ell'=1}^{\ell-1} \frac{5}{n_{\ell'}} + O(n^{-2}).
\]
\noindent \textbf{Exercise.} Redo this analysis for $\sigma(t)=\tanh(t)$ to find that if $n_\ell\equiv n$ we have
\[
\frac{\kappa_{4}^{(\ell)}}{\lr{K_{\alpha\alpha}^{(\ell)}}^2} = \frac{2\ell}{3n}\lr{1 + o_\ell(1)} +O_\ell(n^{-2}).
\]
Hint: you should start by deriving the asymptotics
\[
K_{\alpha\alpha}^{(\ell)} = \frac{1}{2\ell}+O(\log(\ell)/\ell^2),
\]
using this to compute the form of the coefficients in the recursion for $\kappa_4^{(\ell)}$, and then solve this recursion to leading order in $\ell$. 

\section{Lecture 5}\label{S:L5}

\subsection{Introduction}
As in the last lecture, let us fix $L\geq 1$, $n_0,\ldots, n_{L+1}\geq 1$, and $\sigma:\R\gives \R$. We will continue to consider a fully connected feed-forward network, which to an input $x_\alpha\in \R^{n_0}$ associates an output $z_\alpha^{(L+1)}\in \R^{n_{L+1}}$ as follows:
\begin{equation}\label{E:z-def2}
    z_{i;\alpha}^{(\ell+1)}=\begin{cases} \sum_{j=1}^{n_\ell} W_{ij}^{(\ell+1)}\sigma\lr{z_{j;\alpha}^{(\ell)}},&\quad \ell \geq 1\\
    \sum_{j=1}^{n_0} W_{ij}^{(1)}x_{j;\alpha},&\quad \ell=0
    \end{cases}
\end{equation}
Note that we have set the biases to be $0$. We will mainly be interested in the setting where
\[
n_1,\ldots, n_L \simeq n \gg 1
\]
and we have tuned to criticality: 
\[
W_{ij}^{(\ell+1)}= \sqrt{\frac{2}{n_\ell}}\widehat{W}_{ij}^{(\ell+1)} ,\qquad \widehat{W}_{ij}^{(\ell+1)}\sim \mu ,\qquad b_{i}^{(\ell+1)}= 0,
\]
where $\mu$ is any distribution of $\R$ with:
\begin{itemize}
    \item $\mu$ is symmetric around $0$ with no atoms
    \item $\mu$ has variance $1$ and finite (but otherwise arbitrary) higher moments.
\end{itemize}

\subsection{Goal} 

The goal of this lecture is to introduce a  combinatorial formalism for studying the important special case of ReLU network at a single input:
\[
\sigma(t)=\mathrm{ReLU}(t)=t{\bf 1}_{t>0},\qquad x_\alpha\neq 0\in \R^{n_0}\text{ fixed}.  
\]
The main results will illustrate the following

\begin{theorem}[Meta-Claim]\label{T:relu} The behavior of a random ReLU network with iid random weights at a single input is exactly solvable and is determined by the inverse temperature
\[
\beta:=5\lr{\frac{1}{n_1}+\cdots + \frac{1}{n_L}}\simeq \frac{5L}{n}.
\]
Specifically, 
\begin{itemize}
    \item The distribution of the squared entries $\lr{\partial_{x_{p;\alpha}}z_{q;\alpha}^{(L+1)}}^{2}$ of the input-out Jacobian are log-normal with inverse temperature $\beta$:
    \[
    \lr{\partial_{x_{p;\alpha}}z_{q;\alpha}^{(L+1)}}^{2} \simeq \exp\left[\mN(-\frac{\beta}{2},\beta)\right]
    \]
    We will derive this result shortly.
    \item The fluctuations of the NTK $\Theta_{\alpha\alpha}^{(L+1)}$ evaluated a single input at initialization are exponential in $\beta:$
    \[
    \E{\Theta_{\alpha\alpha}^{(L+1)}} \sim L,
    \qquad
    \frac{\E{\lr{\Theta_{\alpha\alpha}^{(L+1)}}^2}}{\E{\Theta_{\alpha\alpha}^{(L+1)}}^2} \sim \exp\left[5\beta\right].
    \]
    \item The relative change in the NTK from one step of GD is
    \[
    \frac{\E{\Delta \Theta_{\alpha\alpha}^{(L+1)}}}{\E{\Theta_{\alpha\alpha}^{(L+1)}}} \sim \frac{L}{n}\exp\left[5\beta \right]
    \]
\end{itemize}
\end{theorem}

\subsection{Formalism For Proof of Theorem \ref{T:relu}}
The purpose of this section is to introduce a  combinatorial approach to understanding essentially any statistic of a random ReLU network that depends on its values at a single input. I developed this point of view in the articles \cite{hanin2018neural, hanin2018products, hanin2019finite}. 

To explain the setup let us fix $L\geq 1$ as well as $n_0,\ldots, n_{L+1}\geq 1$ and a random ReLU network $x\in \R^{n_0}\mapsto z^{(L+1)}(x)\in \R^{n_{L+1}}$ defined recursively by
\[
z^{(\ell+1)}(x)=\begin{cases} W^{(1)}x+b^{(1)}\in \R^{n_1},&\quad \ell =0 \\ W^{(\ell+1)}\sigma(z^{(\ell)}(x))+b^{(\ell+1)}\in \R^{n_{\ell+1}},&\quad \ell \geq 1\end{cases}.
\]
We assume that $W^{(\ell)}=(W_{ij}^{(\ell)},\, i=1,\ldots, n_{\ell},\, j=1,\ldots, n_{\ell-1})$ are independent:
\[
W_{ij}^{(\ell)}:=\lr{\frac{2}{n_{\ell-1}}}^{1/2}\widehat{W}_{i,j}^{(\ell)}, \qquad \widehat{W}_{i,j}^{(\ell)}\sim \mu,\, \text{iid},
\]
where $\mu$ is any fixed probability measure on $\R$ that satisfying
\begin{itemize}
    \item $\mu$ has a density $d\mu(x)$ relative to Lebesgue measure.
    \item $\mu$ is symmetric around $0$ in the sense that $d\mu(x)=d\mu(-x)$ for all $x$
    \item $\mu$ has variance $1$ in the sense that $\int_{\R}x^2 d\mu(x)=1$.
\end{itemize}
The key result which allows for a specialized combinatorial analysis of ReLU networks evaluated a single input is the following:

\begin{proposition}[Exact Matrix Model Underlying Random ReLU Networks]\label{P:relu-model}
The values of a random ReLU network at init evaluated at a single input is equal in distribution to a deep linear network with dropout $p=1/2$:
\[
z_\alpha^{(L+1)} \stackrel{d}{=} W^{(L+1)}D^{(L)}W^{(L)}\cdots D^{(1)}W^{(1)}x_\alpha,
\]
where
\[
D^{(\ell)}=\mathrm{Diag}\lr{\xi_1,\ldots, \xi_{n_\ell}},\quad \xi_i \sim \mathrm{Bernoulli}(1/2).
\]
\end{proposition}
\begin{proof}[Sketch of Proof]
We always have
\[
z_\alpha^{(L+1)} \stackrel{d}{=} W^{(L+1)}\widehat{D}_\alpha^{(L)}W^{(L)}\cdots \widehat{D}_\alpha^{(1)}W^{(1)}x_\alpha,
\]
where
\[
\widehat{D}_\alpha^{(\ell)}=\mathrm{Diag}\lr{{\bf 1}_{\set{z_{i;\alpha}^{(\ell)}>0}},\, i=1,\ldots,n_\ell}.
\]
Conditional on $\ell$, the neuron pre-activations $z_{i;\alpha}^{(\ell+1)}$ are independent. Moreover, since the distribution of $W_{ij}^{(\ell+1)}$ is symmetric around $0$, we have
\[
{\bf 1}_{\set{z_{i;\alpha}^{(\ell+1)}>0}} \stackrel{d}{=} \mathrm{Bernoulli}(1/2).
\]
However, this distribution is independent of $z_\alpha^{(\ell)}$ and hence is also the unconditional distribution of the variables ${\bf 1}_{\set{z_{i;\alpha}^{(\ell+1)}>0}} $. This proves that they are independent. Finally, by symmetrizing the signs of all network weights, we have that, on the one hand, the distribution of any function that is even in the networks weights is unchanged and, on the other hand, that the collection ${\bf 1}_{\set{z_{i;\alpha}^{(\ell+1)}>0}} $ runs through all possible values $\set{0,1}^{\#\text{neurons}}$ configurations. Thus, they are independent. 
\end{proof}
\noindent In order to study random ReLU networks we will make use of the following notation.
\begin{definition}
The space of \textit{paths} in a ReLU network with layer widths $n_0,\ldots,n_{L+1}$ is 
\[
\Gamma:=[n_0]\times \cdots \times [n_{L+1}],
\]
where for any $n\geq 1$ we have $[n]=\set{1,\ldots, n}$. A path $\gamma=\lr{\gamma(0),\ldots, \gamma(L+1)}\in \Gamma$ determines weights and pre-activations:
\[
W_\gamma^{(\ell)}:= W_{\gamma(\ell-1)\gamma(\ell)}^{(\ell)},\qquad z_{\gamma;\alpha}^{(\ell)}:=z_{\gamma(\ell);\alpha}^{(\ell)}.
\]
\end{definition}
\noindent These paths are useful because of the following well-known formula
\begin{equation}\label{E:}
z_{q;\alpha}^{(L+1)}:=\sum_{p=1}^{n_0}x_{p;\alpha}\sum_{\gamma\in \Gamma_{p,q}} W_{\gamma}^{(\ell+1)}\prod_{\ell=1}^{L}W_{\gamma}^{(\ell)}\xi_{\gamma;\alpha}^{(\ell)},\qquad \xi_{\gamma;\alpha}^{(\ell)}:= {\bf 1}_{\set{z_{\gamma;\alpha}^{(\ell)}>0}}.    
\end{equation}
\noindent \textbf{Exercise.} Check that this formula is valid. \\

\noindent Note that Proposition \ref{P:relu-model} allows us to assume
\[
\xi_{\gamma;\alpha}^{(\ell)}\sim \mathrm{Bernoulli}(1/2)\,iid.
\]
\subsection{Formulas for Gradients Using Paths}
In this section, we will record, in the form of exercises, some formulas for gradients.\\

\noindent \textbf{Exercise.} Show that
\begin{align*}
    \frac{\partial z_{q;\alpha}^{(L+1)}}{\partial x_{p;\alpha}} = \sum_{\gamma\in \Gamma_{p,q}} W_{\gamma}^{(L+1)}\prod_{\ell=1}^{L}W_{\gamma}^{(\ell)}\xi_{\gamma;\alpha}^{(\ell)}.
\end{align*}
Conclude that the distribution of $ \partial z_{p;\alpha}^{(L+1)}/\partial x_{q;\alpha} $ is the same for all $x_\alpha\neq 0$.\\

\noindent \textbf{Exercise.} Show that
\[
    \frac{\partial z_1^{(L+1)}(x)}{\partial \widehat{W}_{ij}^{(\ell)}} = \sum_{p=1}^{n_0}x_p  \sum_{\substack{\gamma\in \Gamma_{p,1}\\ \gamma(\ell-1)=j,\, \gamma(\ell)=i}} \lr{\frac{C_W}{n_{\ell-1}}}^{1/2}\frac{W_{\gamma}^{(L+1)}\prod_{\ell'=1}^{L}W_{\gamma}^{(\ell')} \xi_{\gamma;\alpha}^{(\ell')}}{W_{ij}^{(\ell)}}
    \]
and hence also that
     \begin{align*}
         &\E{\lr{\frac{\partial z_1^{(L+1)}(x)}{\partial \widehat{W}_{ij}^{(\ell)}}}^2}\\
         &\qquad = \sum_{p_1,p_2=1}^{n_0}x_{p_1,p_2}  \sum_{\substack{\gamma_1,\gamma_2\in \Gamma_{p_1,1},\Gamma_{p_2,1}\\ \gamma_k(\ell-1)=j,\, \gamma_k(\ell)=j,\, k=1,2}} \frac{C_W}{n_{\ell-1}}\E{\prod_{k=1}^2 W_{\gamma_k}^{(L+1)}\frac{\prod_{\ell'=1}^{L} W_{\gamma_k}^{(\ell')} \xi_{\gamma_k;\alpha}^{(\ell')}}{\lr{W_{ij}^{(\ell)}}^2}}.
     \end{align*}
Assume $n_{L+1}=1$ and use this to derive a sum-over-paths formula for the on-diagonal NTK 
\[
\Theta_{\alpha,\alpha}^{(L+1)}=\sum_{\ell=1}^{L+1}\sum_{i=1}^{n_\ell}\sum_{j=1}^{n_{\ell-1}} \lr{\frac{\partial z_{1;\alpha}^{(L+1)}}{\partial W_{ij}^{(\ell)}} }^2.
\]
\subsection{Deriving $L/n$ Behavior of Input-Output Jacobian}
The purpose of this section is to prove that
\begin{align*}
    \E{\lr{\frac{\partial z_{q;\alpha}^{(L+1)}}{\partial x_{p;\alpha}}}^2} = \frac{2}{n_0},\qquad \E{\lr{\frac{\partial z_{q;\alpha}^{(L+1)}}{\partial x_{p;\alpha}}}^4} = \frac{\text{const}}{n_0^2}\exp\left[5\sum_{\ell=1}^{L}\frac{1}{n_\ell}+ O\lr{\frac{L}{n^2}}\right].
\end{align*}
\subsubsection{Second Moment Computation} We will start with the second moment and will do it in several (unnecessarily many) steps to illustrate the general idea we'll need for the fourth moment.
First, note that
\begin{align*}
      \E{\lr{\frac{\partial z_{q;\alpha}^{(L+1)}}{\partial x_{p;\alpha}}}^2} &= \E{ \sum_{\gamma_1,\gamma_2\in \Gamma_{p,q}} \prod_{k=1}^2 W_{\gamma_k}^{(L+1)}\prod_{\ell=1}^{L}W_{\gamma_k}^{(\ell)}\xi_{\gamma_k;\alpha}^{(\ell)}}\\
      &= \sum_{\gamma_1,\gamma_2\in \Gamma_{p,q}}  \E{\prod_{k=1}^2W_{\gamma_k}^{(L+1)}}\prod_{\ell=1}^{L}\E{\prod_{k=1}^2 W_{\gamma_k}^{(\ell)}}\E{\prod_{k=1}^2\xi_{\gamma_k;\alpha}^{(\ell)}}.
\end{align*}
Next, note that
\[
\E{\prod_{k=1}^2\lr{W_{\gamma_k}^{(\ell)}}^2} = \frac{C_W}{n_{\ell-1}}\delta_{\gamma_1(\ell-1)\gamma_2(\ell-1)}\delta_{\gamma_1(\ell)\gamma_2(\ell)}.
\]
In other words, the paths $\gamma_1,\gamma_2$ have to ``collide'' in every layer. In particular, since they have the same starting and ending points, they must agree at all layers. In particular, since
\[
\E{\xi_{\gamma;\alpha}^{(\ell)}}=\frac{1}{2},
\]
we find
\begin{align*}
      \E{\lr{\frac{\partial z_{q;\alpha}^{(L+1)}}{\partial x_{p;\alpha}}}^2} &= \sum_{\gamma\in \Gamma_{p,q}}  \frac{2}{n_L}\prod_{\ell=1}^{L}\frac{2}{n_{\ell-1}}\cdot \frac{1}{2}= 2\prod_{\ell=0}^L \frac{1}{n_\ell} \sum_{\gamma\in \Gamma_{q,p}} 1.
\end{align*}
Note that 
\[
\abs{\Gamma_{q,p}} = \prod_{\ell=1}^{L}n_\ell.
\]
Hence, we may actually re-write
\begin{align*}
      \E{\lr{\frac{\partial z_{q;\alpha}^{(L+1)}}{\partial x_{p;\alpha}}}^2} &= \frac{2}{n_0}\mathcal E\left[1\right],
\end{align*}
where $\mathcal E\left[\cdot \right]$ denotes the expectation operator over the choice of a uniformly random path $\gamma=\lr{\gamma(0),\ldots, \gamma(L+1)}\in \Gamma_{p,q}$ in which every neuron in every layer is chosen uniformly at random:
\[
\gamma(0)=p,\quad \gamma(L+1)=q,\quad \gamma(\ell)\sim \mathrm{Unif}(\set{1,\ldots, n_\ell})\,\, iid.
\]
Finally, since the average of $1$ is $1$, we conclude
\begin{align}
      \E{\lr{\frac{\partial z_{q;\alpha}^{(L+1)}}{\partial x_{p;\alpha}}}^2} &= \frac{2}{n_0},
\end{align}
as desired.

\subsubsection{Fourth Moment Computation} For simplicity, we will add a $\sigma$ to the network output (it only changes things by a factor of $2$). We have
\begin{align*}
      \E{\lr{\frac{\partial \sigma(z_{q;\alpha}^{(L+1)})}{\partial x_{p;\alpha}}}^4} &= \E{ \sum_{\gamma_1,\ldots,\gamma_4\in \Gamma_{p,q}} \prod_{\ell=1}^{L+1}W_{\gamma_k}^{(\ell)}\xi_{\gamma_k;\alpha}^{(\ell)}}\\
      &= \sum_{\gamma_1,\ldots, \gamma_4\in \Gamma_{p,q}}  \prod_{\ell=1}^{L+1}\E{\prod_{k=1}^4 W_{\gamma_k}^{(\ell)}}\E{\prod_{k=1}^4\xi_{\gamma_k;\alpha}^{(\ell)}}.
\end{align*}
Just as in the 2nd moment case, we find that all weights must appear an even number of times, so let us write
\[
\Gamma_{p,q}^{4,even}:=\set{\gamma_1,\ldots, \gamma_4\in\Gamma_{p,q}~|~ \text{$\forall \, \ell$, the multi-set }\set{W_{\gamma_k}^{(\ell)},\, k=1,\ldots,4}\text{ has even multiplicity}}.
\]
Thus, 
\begin{align*}
      \E{\lr{\frac{\partial \sigma(z_{q;\alpha}^{(L+1)})}{\partial x_{p;\alpha}}}^4} &= \E{ \sum_{\gamma_1,\ldots,\gamma_4\in \Gamma_{p,q}} \prod_{\ell=1}^{L+1}W_{\gamma_k}^{(\ell)}\xi_{\gamma_k;\alpha}^{(\ell)}}\\
      &= \sum_{(\gamma_1,\ldots, \gamma_4)\in \Gamma_{p,q}^{4,even}}  \prod_{\ell=1}^{L+1}\E{\prod_{k=1}^4 W_{\gamma_k}^{(\ell)}}\E{\prod_{k=1}^4\xi_{\gamma_k;\alpha}^{(\ell)}}.
\end{align*}
Let us now define the collision events
\[
C^{(\ell)}=C^{(\ell)}(\gamma_1,\ldots, \gamma_4):=\set{\gamma_1(\ell)=\cdots = \gamma_4(\ell)}.
\]
Thus, 
\[
\E{\prod_{k=1}^4 W_{\gamma_k}^{(\ell)}}\E{\prod_{k=1}^4\xi_{\gamma_k;\alpha}^{(\ell)}} = \frac{1}{n_{\ell-1}^2}\lr{1+{\bf 1}_{C^{(\ell)}}+{\bf 1}_{C^{(\ell)}C^{(\ell-1)}}2(\mu_4-1)} \lr{1+\delta_{\ell1}},
\]
where
\[
\mu_4 = \int_\R x^4 d\mu(x)
\]
and we have made use of the fact that ${\bf 1}_{C^{(\ell)}} {\bf 1}_{C^{(\ell)}C^{(\ell-1)}} = {\bf 1}_{C^{(\ell)}C^{(\ell-1)}}$.
Putting this all together yields
\begin{align*}
      \E{\lr{\frac{\partial \sigma(z_{q;\alpha}^{(L+1)})}{\partial x_{p;\alpha}}}^4} 
      &= 2\lr{\prod_{\ell=1}^{L+1}\frac{1}{n_{\ell-1}^2}}\sum_{(\gamma_1,\ldots, \gamma_4)\in \Gamma_{p,q}^{4,even}}  \prod_{\ell=1}^{L+1}\lr{1+{\bf 1}_{C^{(\ell)}}+{\bf 1}_{C^{(\ell)}C^{(\ell-1)}}2(\mu_4-1)}.
\end{align*}
The trick is now to to change variables in the sum from four paths with even numbers of weights to 2 paths.\\

\noindent \textbf{Exercise}. Given $\gamma_1',\gamma_2'\in \Gamma_{p,q}$ show there are exactly 
\[
6^{\#\text{collisions}}=6^{\sum_{\ell=1}^{L+1} {\bf 1 }_{C^{(\ell)}}}
\]
collections $(\gamma_1,\ldots, \gamma_4)\in \Gamma_{p,q}^{4,even}$ which give rise to the same weight configurations (but doubled). \\

We therefore find
\begin{align*}
      \E{\lr{\frac{\partial \sigma(z_{q;\alpha}^{(L+1)})}{\partial x_{p;\alpha}}}^4} 
      &= \frac{2}{n_0^2}\lr{\prod_{\ell=1}^{L}\frac{1}{n_{\ell}^2}}\sum_{\gamma_1,\gamma_2\in \Gamma_{p,q}}  \prod_{\ell=1}^{L+1}\lr{1+5{\bf 1}_{C^{(\ell)}}+{\bf 1}_{C^{(\ell)}C^{(\ell-1)}}6(\mu_4-3)}\\
      &=\frac{2}{n_0^2} \mathcal E\left[ \prod_{\ell=1}^{L+1}\lr{1+5{\bf 1}_{C^{(\ell)}}+{\bf 1}_{C^{(\ell)}C^{(\ell-1)}}6(\mu_4-3)}\right],
\end{align*}
where the expectation $\mathcal E$ is now over the choice of two iid paths $\gamma_1,\gamma_2\in \Gamma_{p,q}$ in which every neuron in very layer is selected uniformly:
\[
\gamma_k(\ell)\sim \mathrm{Unif}\lr{\set{1,\ldots, n_\ell}}\quad \text{iid}.
\]
Finally, we'll be a bit heuristic: note that
\[
\mathcal P[C^{(\ell)}] = \frac{1}{n_\ell},\qquad \mathcal P[C^{(\ell-1)},\, C^{(\ell)}] \approx \frac{1}{n_\ell n_{\ell-1}}= O(n^{-2}).
\]
In particular, we find approximately 
\begin{align*}
      \E{\lr{\frac{\partial \sigma(z_{q;\alpha}^{(L+1)})}{\partial x_{p;\alpha}}}^4} \approx \frac{12}{n_0^2} \mathcal E\left[ \prod_{\ell=1}^{L}\lr{1+\frac{5}{n_\ell} +O(n^{-2})}\right]=\frac{12}{n_0}\exp\left[5\sum_{\ell=1}^L \frac{1}{n_\ell} + O\lr{\frac{L}{n^2}}\right],
\end{align*}
as desired.\\

\noindent \textbf{Exercise.} Make the reasoning in the previous computation precise. 

\subsubsection{Analogous calculations for the NTK}

This concludes our analysis of the input-output Jacobian (i.e. the derivative of the network output with respect to the input).
Virtually the same analysis can be applied to instead study the parameter-output Jacobian (i.e. the derivative of the network output with respect to a particular parameter), which will then, with a sum over all parameters, yield the NTK at a particular input.
Similar calculations to those worked above quickly give the 2nd and 4th moments of the NTK.
These yield that the partial derivatives of the output with respect to each parameter are independent (i.e. they have zero covariance), and so the mean NTK is simple and given by its infinite-width value, while the fluctuations about that mean scale as $e^\beta - 1$.

\subsection{Open questions and dreams}

We have shown that various quantities of interest are straightforwardly calculable for random ReLU networks with a single input vector.
We conclude with various open questions regarding related quantities.

\begin{enumerate}
    \item How smooth is the random function at initialization? Can one obtain a Lipschitz constant? Is it adversarially attackable?
    \item Can we obtain a clear picture of feature learning (even at a single step) beyond the mere fact that it occurs? How does this feature learning relate to those of e.g. mean-field networks \cite{mei:2018-mean-field} or the large-learning-rate regime of the NTK regime discussed in Yasaman's lectures?
    \item Can we profitably perform similar path-counting arguments with activation functions besides ReLU?
\end{enumerate}



\bibliography{SciPost_Example_BiBTeX_File.bib}

\nolinenumbers

\end{document}